\documentclass[11pt]{article}
\usepackage[OT1]{fontenc}
\usepackage{amsmath,amsfonts,amssymb}
\usepackage{mathtools, enumitem}
\usepackage{amsthm} 
\usepackage{latexsym}
\usepackage{relsize}
\usepackage{thmtools,thm-restate}
\usepackage{thm-restate}
\usepackage{siunitx}
\usepackage{xcolor}
\usepackage{fullpage}
\usepackage{natbib}
\usepackage{algorithm,algorithmic}
\usepackage[colorlinks=true, linkcolor=blue, urlcolor=blue, citecolor=blue]{hyperref}
\usepackage{etoolbox}

\newtoggle{arxiv}
\toggletrue{arxiv}

\newcommand{\F}{\mathcal{F}}
\newcommand{\I}{\mathcal{I}}

\def\mL{{\mathcal L}}

\def\X{{\mathcal X}}
\def\H{{\mathcal H}}

\def\R{{\mathbb R}}
\def\N{{\mathcal N}}

\numberwithin{equation}{section}

\newcommand{\A}{\mathcal{A}}

\newcommand{\K}{\ensuremath{\mathcal K}}

\def\frelu{\sigma_{\text{relu}}}

\def\regret{\mbox{{Regret}}}
\def\relu{\mbox{{ReLU }}}
\def\sphere{\mathbb{S}}

\newcommand{\ignore}[1]{}

\theoremstyle{plain}
\newtheorem{theorem}{Theorem}[section]
\newtheorem{lemma}{Lemma}[section]

\newtheorem{corollary}{Corollary}[section]

\newtheorem{assumption}{Assumption}
\newtheorem{definition}{Definition}[section]

 \theoremstyle{plain}

\DeclareMathAlphabet{\mathbfsf}{\encodingdefault}{\sfdefault}{bx}{n}

\DeclareMathOperator*{\argmin}{arg\,min}

\let\Pr\relax
\DeclareMathOperator{\Pr}{\mathbb{P}}

\newcommand{\E}{\mathbb{E}}

\newcommand{\poly}{\mathrm{poly}}

\newcommand{\reals}{\mathbb{R}}
\newcommand{\eps}{\varepsilon}

\renewcommand{\leq}{~\le~}
\renewcommand{\geq}{~\ge~}

\let\oldtfrac\tfrac
\renewcommand{\tfrac}[2]{\smash{\oldtfrac{#1}{#2}}}

\let\nablaold\nabla
\renewcommand{\nabla}{\nablaold\mkern-2.5mu}

\title{\huge Provable Regret Bounds \\
for Deep Online Learning and Control}

\author{
Xinyi Chen$^{1\, 3}$\hspace{3em}  Edgar Minasyan$^{1\, 3}$  \hspace{3em}   Jason D. Lee$^{2\, 3}$ \hspace{3em} Elad Hazan$^{1\,3}$\\
  $^1$ Department of Computer Science, Princeton University\\ $^2$ Department of Electrical Engineering, Princeton University\\ $^3$ Google AI Princeton \\
  \texttt{\{xinyic, minasyan, jasonlee, ehazan\}@princeton.edu}
	}

\ifdefined\usebigfont

\usepackage{times}
\usepackage[fontsize=13pt]{scrextend}
\usepackage{geometry}
\AtBeginDocument{
	\newgeometry{letterpaper,left=1.56in,right=1.56in,top=1.71in,bottom=1.77in}
}
\else
\fi

\begin{document}

\maketitle

\begin{abstract}

The theory of deep learning focuses almost exclusively on supervised learning, non-convex optimization using stochastic gradient descent, and overparametrized neural networks. It is common belief that the optimizer dynamics, network architecture, initialization procedure, and other factors tie together and are all components of its success. This presents theoretical challenges for analyzing state-based and/or online deep learning.

Motivated by applications in control, we give a general black-box reduction from deep learning to online convex optimization. This allows us to decouple optimization, regret, expressiveness, and derive agnostic online learning guarantees for fully-connected deep neural networks with ReLU activations. We quantify convergence and regret guarantees for any range of parameters and allow any optimization procedure, such as adaptive gradient methods and second order methods.
As an application, we derive provable algorithms for deep control in the online episodic setting.



\end{abstract}

\ignore{
\subsection{what we do}
\begin{enumerate}
    \item main: online control with dnn, provable regret bounds. this wasn't done before. 
    
    \item for this, we need:\\
    a. online learning of dnn, general losses. \\
    b. reduction from min regret of control to online learning w. memory, and then to online learning. OR, we do episodic, which is just online learning. 

    \item we define a notion of expressivity and give regret bounds that depend on the given expressivity specification.   we relate this expessivity to online learning snd control.

    \item previous work: \\
    sanjeev+yi: they had online agnostic, two layer. not a general loss. different setting.
    \\
    gao et al: many layer, but no explicit regret bound. different setting. 

\end{enumerate}
}

\section{Introduction}

Deep learning has transformed the field of supervised learning and is becoming increasingly dominant in  reinforcement learning and control. However, a theory for deep control and reinforcement learning (RL) remains challenging. 
The main difficulty in applying the theory developed for supervised learning to the RL domain is the distributional assumptions and realizability goal made in the literature thus far. In control and RL, the environment is inherently online and often nonstochastic, and the goal is usually agnostic learning with respect to a policy class. 


With this motivation in mind, we propose a black-box reduction from deep learning to online convex optimization  that attain provable regret bounds. These bounds apply to the general online learning setting with high dimensional output predictors and any convex loss function. Moreover, the regret guarantees are {\it agnostic}, i.e. they show competitive performance to the best neural network in hindsight without assuming it achieves zero loss. To capture this notion of agnostic learning and derive meaningful guarantees for online learning and control, we also introduce a new metric of expressivity, namely the ``interpolation dimension", that accompanies our regret bounds.

An interesting conclusion from this reduction is that provable bounds for training deep neural networks can be derived from any OCO method, beyond online or stochastic gradient descent. This includes mirror descent, adaptive gradient methods, follow-the-perturbed leader and other algorithms. Previously, convergence and generalization analyses for neural networks were done in isolation for different optimization algorithms as detailed in the related work section below. 

Equipped with this reduction, we the results for deep online learning to obtain provable regret guarantees for  control of dynamical systems with deep neural networks. Provable regret bounds in this domain have thus far been limited to linear dynamics and/or linear controllers. However, most dynamical systems in the physical world are nonlinear and require nonlinear controls.  

An important tool that allows us to go beyond linear control is the emerging paradigm of online nonstochastic control: a  methodology for control that is robust to adversarial noise in the dynamics. The important aspect of this paradigm to our study is that it uses policy classes that admit a convex parameterization.  Therefore, our extension of OCO to deep neural networks naturally leads to regret bounds for deep neural network controllers in this setting. 

Our contributions can be summarized as follows:
\begin{itemize}

    \item We define a parametrized expressivity metric for statistical and online learning. This metric is later used to give a precise relation between regret and expressivity for deep online learning.

    \item We give a general reduction from any online convex optimization algorithm to online deep learning. The regret bounds obtained depend on the original regret for online convex optimization, the width of the network, and the new expressivity metric. The precise statement is given in Theorem 
    \ref{thm:main_nn}.
    
    \item These regret bounds imply generalization in the statistical setting, and go beyond SGD: any online convex optimization algorithm can be shown to generalize over deep neural networks according to its regret bound in the OCO framework. This includes commonly used algorithms such as Adagrad/Adam.
    
    \item We apply this reduction to the framework of online nonstochastic control, and obtain provable regret bounds for deep controllers in the episodic  setting. These bounds, precisely stated in Theorem \ref{thm:control}, can be used to derive guarantees for first-order trajectory-based control/planning over \emph{nonlinear} dynamics.

\end{itemize}

\subsection{The interpolation dimension: an expressivity notion}

Characterizing the expressive power of the hypothesis class of deep neural networks is an area of investigation spanning numerous research papers from multiple communities \citep{malach2021connection, yehudai2019power, rahimi2008uniform} . The literature mostly focuses on differentiating the expressivity of networks according to their depth, or on proving lower bounds for sample complexity. 
Our focus, however, is different. We prove regret bounds for deep learning with families of deep neural networks as the comparator classes. Thus, we need to  ensure that these families are non-trivial in terms of their representation power. 

It is useful to recall the theory of supervised learning for binary classification. Vapnik's theorem asserts that a statistical learning problem is learnable if and only if its VC dimension is finite, and gives precise sample complexity bounds in terms of this dimension. For many common examples of hypothesis classes, the VC dimension also characterizes their expressive power. For example, linear classifiers of dimension $k$ are capable of shattering {\it any} training set of size $k+1$, as long as it is non-degenerate (and in this case, linearly independent). This is, however, not a requirement of the VC dimension, which only requires the {\it existence} of a set that can be shattered by the hypothesis class. It is thus useful to consider the dual of the VC dimension, which we formally define as follows:
\begin{definition}\label{def:interpolation_dim}
The {\bf interpolation dimension} of a hypothesis class $\H = \{h : \X \mapsto \{0,1\} \}$, denoted $\mathcal{I}(\H)$, is the largest cardinality $k$ such that for {\bf any non-degenerate} set of examples $x_1,...,x_k \in \X$, and any labels $y_1,...,y_k \in \{0,1\}$, the mapping $h(x_i)=y_i, \forall i \in [1, k]$ can be expressed by a hypothesis $h \in \H$. 
\end{definition}

Clearly $\mathcal{I}(\H) \leq \mathrm{VC}(\H)$, and it is equal for many common examples, including linear separators (when the examples are linearly independent). 
Notice that the non-degeneracy assumption is necessary to avoid non-separability due to ``trivial" reasons, such as having two different labels assigned to the same example. 
The importance of this non-degeneracy is even more pronounced in the definition's generalization to real-valued learning we now consider.

In a nutshell, we say that a hypothesis class has interpolation dimension of at least $k$ if one can assign arbitrary real labels to \emph{any} $k$ different inputs using a hypothesis from that class. As a generalization of the binary case, the formal definition is presented for multi-dimensional and real valued outputs that are necessary for applications such as control:
\begin{definition}\label{def:interpolation_realdim}
The interpolation dimension of a hypothesis class $\H = \{h: \X \subseteq \R^p \to \R^d\}$ at non-degeneracy $\gamma>0$, denoted $\I_{\gamma}(\H)$, is the largest cardinality $k$ such that for any set of data points $\{(x_j, y_j)\}_{j=1}^k$ satisfying $\min_{j\neq l} \|x_j - x_l\|_2 \ge \gamma,\ y_j\in [-1, 1]^d$ for all $j\in [k]$, and for any $\eps > 0$, there exists a hypothesis $h\in \H$ s.t. 
\begin{align*}
\sum_{j=1}^k \|y_j - h(x_j)\|_2^2 \le \eps.
\end{align*}
\end{definition}

The label bound is chosen as $1$ for simplicity but can be extended to any $B > 0$. 
Henceforth, we show that, under input domain $\X = \mathbb{S}_p$, neural networks that have $\poly(k,\frac{1}{\gamma})$ width have interpolation dimension $\I_{\gamma}(\H) \geq k$. This enables us to derive regret bounds for online agnostic learning over neural networks with respect to a benchmark hypothesis class of interpolation dimension at least $k$. 
To illustrate this definition of expressivity, its relation to the VC dimension and its applicability, consider the following examples:
\begin{enumerate}
    \item The hypothesis class of linear hyperplanes of dimension $k$ (with bias) has interpolation dimension $k+1$ for any $\gamma>0$ non-degeneracy \emph{if} the input domain $\X$ is restricted to linearly independent points. In this case, the VC dimension of this class is also equal to $k+1$. However, if the input domain $\X$ is \emph{not} restricted to linearly independent points, the interpolation dimension of this hypothesis class equals $2$ for any $\gamma > 0$ non-degeneracy while the VC dimension remains unchanged. 
    \iftoggle{arxiv}{
    }{}

    \item
Consider online learning of a Boolean function $\{0,1\}^{\log k} \mapsto \{0,1\}$. There are $2^{2^{\log k}} = 2^k$ such functions, and running an experts algorithm such as Hedge on all possible such functions results in $O(\sqrt{T k})$ regret, but requires maintaining $2^k$ weights on the experts. 
On the other hand, online learning of a deep network (or any hypothesis class that can be learned efficiently) that has interpolation dimension $k$ can learn the same class of functions, with regret that is $\poly(k) \sqrt{T}$ and $\poly(k)$ running time.

Note that it is possible to learn this problem efficiently using an experts algorithms on the possible entries of the truth table of these functions. 

\end{enumerate}

\subsection{Related work}

\paragraph{Online convex optimization and dimensionality notions in learning.} The framework of learning in games has been extensively studied as a model for learning in adversarial and nonstochastic environments  \citep{cesa2006prediction}. 
Online learning was infused with algorithmic techniques from mathematical optimization into the setting of  online convex optimization, see \citep{hazan2019introduction} for a comprehensive  introduction. 
Learnability in the statistical and online learning settings was characterized using various notions of dimensionality, starting from the VC-dimension, fat-shattering dimension, Rademacher complexity, Littlestone dimension and more. For an extensive treatment see \citep{mohri2018foundations,shalev2014understanding,vapnik1999nature}. Regarding interpolation, \citet{bubeck2021a} establish an inverse relationship between the interpolation ability and robustness of a function class. 

\paragraph{The emerging theory of deep learning.}
For detailed background on the developing theory for deep learning, see the book draft~\citep{arorabook}. Among the various studies on the theory of deep learning, the neural tangent kernel or linearization approach has emerged as the most general and pervasive. This technique shows that neural networks behave similar to their linearization and proves that gradient descent converges to a global minimizer of the training loss~\citep{soltanolkotabi2018theoretical,du2018gradient_deep,du2018gradient,allenzhu2019convergence,zou2018stochastic,jacot2018neural,bai2019beyond}. Techniques related to this have been expanded to provide generalization error bounds in the i.i.d. statistical setting~\citep{arora2019fine,wei2019regularization}, and generalization bounds for SGD~\citep{cao2019generalization,ji2020polylogarithmic}. As opposed to our generic approach, several different optimization algorithms were considered in isolation for deep learning theory, see e.g. \citep{wu2019global,cai2019gram,wu2019global,wu2021adaloss}. 


\paragraph{Online-to-batch linearization.} The linearization technique has been combined with online learning and online-to-batch conversion to yield generalization bounds for SGD \citep{cao2019generalization}. In the adversarial training setting, the online gradient proof technique was used in \citet{gao2019convergence, zhang2020overparameterized} to handle non i.i.d. functions.
In relation to these works: (i) our reduction takes in a general OCO algorithm, rather than only OGD/SGD; (ii) we generalize to the full OCO model, including high dimensional output predictors and general convex costs to enable the application to control; (iii) we use the OCO framework against an adversary repeatedly providing new data points as in the standard online learning setting while in the adversarial training setting, it is used against an adversary repeatedly perturbing a fixed training set.

\paragraph{Online and nonstochastic control.}
Our study focuses on algorithms which enjoy sublinear regret for online control of dynamical systems; that is, whose performance tracks a given benchmark of policies up to a term which is vanishing relative to the problem horizon. \citet{abbasi2011regret} initiated the study of online control under the regret benchmark for linear time-invariant (LTI) dynamical systems. Bounds for this setting have since been improved and refined in \citet{dean2018regret,mania2019certainty,cohen2019learning,simchowitz2020naive}.
Our work instead adopts the \emph{nonstochastic} control setting \citep{agarwal2019online}, that allows for adversarially chosen (e.g. non-Gaussian) noise and general convex costs that may vary with time. This model has been studied for many extended settings, see \cite{hazan2021tutorial} for a comprehensive survey. Similar to our control framework, online episodic control is also studied in \citet{kakade2020information}, but the regret definition differs from ours, the results are only information-theoretic and the system is linear in a kernel space rather than simply linear. 

\paragraph{Nonlinear systems and deep neural network based controls.} Nonlinear control is computationally intractable in general \citep{blondel2000survey}. One approach to deal with the computational difficulty is  iterative linearization, which takes the local linear approximation via the gradient of the nonlinear dynamics. One can apply techniques from optimal control to solve the resulting changing linear system. Iterative planning methods such as iLQR \citep{iLQR}, iLC \citep{moore2012iterative} and iLQG \citep{todorov2005generalized} fall into this category. Neural networks were also used to directly control the dynamical system since the 90s, for example in \citet{LEWIS1997161}. More recently, deep neural networks were used in applications of a variety of control problems, including \citet{lillicrap2016continuous} and  \citet{levine2016end}. A critical study of neural-network based controllers vs. linear controllers appears in \citet{rajeswaran2017towards}. 
\iftoggle{arxiv}{
}{}

\section{Preliminaries}\label{sec:prelim}

\subsection{Deep neural networks and the interpolation dimension}\label{sec:dnn_def}

\paragraph{Deep neural networks.} Let $x\in \reals^p$ be the $p$-dimensional input. We define the depth $H$ network with \relu activation and scalar output as follows: 
\begin{align*}
    x^0 &= Ax\\
    x^h &= \frelu(\theta^h x^{h-1}),\ h\in [H]\\
    f(\theta, x) &= a^\top x^{H},
\end{align*}
where $\frelu(\cdot)$ is the \relu function $\frelu(z) = \max(0, z)$, $A\in \reals^{m\times p}$, $\theta^{h}\in \reals^{m\times m}$, and $a \in \reals^{m}$. Let $\theta = (\theta^{1}, \ldots, \theta^{H})^\top \in \reals^{H \times m \times m}$ denote the trainable parameters of the network and the parameters $A, a$ are fixed after initialization. The initialization scheme is as follows: each entry in $A$ and $\theta^{h}$ is drawn i.i.d. from the Gaussian distribution $\N(0, \frac{2}{m})$, and each entry in $a$ is drawn i.i.d. from $\N(0, 1)$. This setup is common in recent literature and follows that of \citet{gao2019convergence}.

For vector-valued outputs, we consider a scalar output network for each coordinate. Suppose for $i\in [d]$, $f_i$ is a deep neural network with a scalar output; with a slight abuse of notation, for input $x\in \reals^p$, denote 
\begin{align}\label{eq:deep_nn} f(\theta; x) = (f_1(\theta[1]; x), \dots, f_d(\theta[d]; x))^\top \in \reals^d,
\end{align}
where $\theta[i] \in \reals^{H\times m\times m}$ denotes the trainable parameters for the network $f_i$ for coordinate $i$. Let $\theta = (\theta[1], \theta[2], \ldots, \theta[d]) \in \R^{d \times H \times m \times m}$ denote all the parameters for $f$.

In the online setting, the neural net receives an input $x_t \in \R^p$ at each round $t\in [T]$, and with parameter $\theta$ suffers loss $\ell_t(f(\theta; x_t))$. Note that this framework generalizes the supervised learning paradigm: with data points $\{(x_t, y_t)\}_{t}$ the losses $\ell_t(\cdot) = \ell(\cdot, y_t)$ reduce the setting to supervised learning. We make the following standard assumptions on the inputs and the loss functions:
\begin{assumption} \label{assumption:unit_norm}
The input $x$ has unit norm, i.e. $x \in \sphere_p$, $\|x\|_2 =1$. 
\end{assumption}
\begin{assumption} \label{assumption:loss}
The loss functions $\ell_t(f(\theta; x))$ are $L$-Lipschitz and convex in $f(\theta; x)$.
\end{assumption}

\paragraph{Interpolation dimension.} As defined in Definition \ref{def:interpolation_realdim}, to have interpolation dimension $k$, a function class must be able to fit arbitrary labels to \emph{any} $k$ inputs. As opposed to the VC dimension that characterizes learnability, the interpolation dimension quantifies the representation power of a given hypothesis class. Regret bounds in the online learning framework, on the other hand, provide the learnability/generalization guarantee. 

One way of characterizing the expressivity of neural networks is by establishing equivalence to a different function class, e.g. the common NTK approach \citep{jacot2018neural}. In contrast, our approach allows for end-to-end online \emph{agnostic} guarantees isolated to, in this case, neural networks. The class of neural networks, defined in \eqref{eq:deep_nn}, is determined by the parameter set $\theta \in \Theta$ which is a ball around initial parameters $\Theta = B(R; \theta_1) = \{ \theta : \| \theta[i] - \theta_1[i] \|_F \leq R, \forall i \in [d] \}$. We establish that, based on the seminal work \citet{allenzhu2019convergence}, with sufficient width $m = \poly(k, \frac{1}{\gamma})$, appropriate radius $R$ and input domain $\X = \sphere_p$ according to Assumption \ref{assumption:unit_norm}, the class of neural networks has interpolation dimension of at least $k$.

\begin{lemma}\label{lem:nn_interpol}
Let $\H_{\mathrm{NN}}(R; \theta_1) = \{ f(\theta; \cdot) : \theta \in \Theta \}$ denote the class of neural networks as in \eqref{eq:deep_nn} where $\Theta = B(R; \theta_1)$ and $\X = \sphere_p$.
Suppose $\gamma \in \left(0, O(\frac{1}{H})\right]$, $m \geq \Omega \left( \frac{k^{24} H^{12} \log^5 m}{\gamma^8} \right)$ and take $R = O\left(\frac{k^3 \log m}{\gamma \sqrt{m}}\right)$, then with probability $1 - d \cdot e^{-\Omega(\log^2 m)}$ over random initialization of $\theta_1$,
\begin{align}\label{eq:nn_interpol_dim}
    \I_{\gamma} \left(\H_{\mathrm{NN}}(R; \theta_1) \right) \geq k ~.
\end{align}
\end{lemma}

\subsection{Online convex optimization}\label{sec:prelim_oco}

In Online Convex Optimization (OCO), a decision maker sequentially chooses a point in a convex set $\theta_t \in \K \subseteq \reals^d$, and suffers loss $\ell_t(\theta_t)$ according to a convex loss function $\ell_t : \K \mapsto \reals$. The goal of the learner is to minimize her regret, defined as 
\begin{align*}
\regret_T = \sum_{t=1}^T \ell_t(\theta_t) - \min_{\theta^* \in \K} \sum_{t=1}^T \ell_t(\theta^*) ~.
\end{align*}
A host of techniques from mathematical optimization are applicable to this setting and give rise to efficient low-regret algorithms. To name a few methods, Newton's method, mirror descent, Frank-Wolfe and follow-the-perturbed leader all have online analogues, see e.g. \citet{hazan2019introduction} for a comprehensive treatment. 

As an extension to the OCO framework, we show that regret bounds hold analogously for the online optimization of \emph{nearly} convex functions. As we show in later sections, these regret bounds naturally carry over to the setting of online learning over neural networks.

\begin{definition}\label{def:nearly_convex} A function $\ell : \reals^n \to \reals$ is $\varepsilon$-nearly convex over the convex, compact set $\K \subset \R^n$ if and only if
\begin{align}\label{eq:nearly_convex}
    \forall x, y \in \K, \ \ell(x) \ge \ell(y) + \nabla \ell(y)^\top (x-y) - \varepsilon ~.
\end{align}
\end{definition}

The analysis of any algorithm for OCO, including the most fundamental method OGD, extends to this case in a straightforward manner. Let $\A$ be any algorithm for OCO with the following regret bounded by $\text{Regret}_T(\A)$. This algorithm $\A$ can be applied on the surrogate loss functions $h_t(\theta) = \ell_t(\theta_t) + \nabla \ell_t(\theta_t)^{\top} (\theta-\theta_t)$ to obtain regret bounds on the nearly convex losses $\ell_t$ as given below. The described method is presented in Algorithm \ref{alg:onco} which along with more details can be found in Appendix \ref{sec:app_onco}.
\begin{lemma}\label{lem:ogd_almost_convex}
Suppose $\ell_1, \ldots, \ell_T$ are $\eps$-nearly convex, then Algorithm \ref{alg:onco} has regret bounded by
\begin{align*}
\sum_{t=1}^T \ell_t(\theta_t)  - \min_{\theta^* \in \K} \sum_{t=1}^T \ell_t(\theta^*) \le \regret_T(\A)  + \eps T ~.
\end{align*}
\end{lemma}


\iftoggle{arxiv}{
}
{}

\section{Deep online learning: main result}\label{sec:olnn}
In this section, we consider the general framework of online learning with deep neural networks and state the accompanying regret guarantees. Our framework can use any OCO algorithm as a black-box as presented in Algorithm \ref{alg:ogd_nn}: for our main result, we use projected Online Gradient Descent (OGD) as the OCO algorithm to provide explicit regret bounds, since it is widely used in practice. Observe that, in this case, the parameter update is equivalent to OGD on the original losses.


\iftoggle{arxiv}
{
\begin{algorithm}
\caption{OGD for Neural Networks}\label{alg:ogd_nn}
\begin{algorithmic}[1]
\STATE \textbf{Input:} step size $\eta_t > 0$, initial parameter $\theta_1$, parameter set $\Theta = B(R; \theta_1)$.
\FOR{$t=1\dots T$}
\STATE Play $\theta_t$, receive loss  $\ell_t(\theta) = \ell_t(f(\theta; x_t))$.
\STATE Construct $h_t(\theta) = \ell_t(\theta_t) + \nabla \ell_t(\theta_t)^\top(\theta - \theta_t)$.
\STATE Update $\theta_{t+1} = \prod_{B(R)} (\theta_t - \eta_t \nabla h_t(\theta_t)) = \prod_{B(R)} (\theta_t - \eta_t \nabla \ell_t(\theta_t))$.
\ENDFOR
\end{algorithmic}
\end{algorithm}
}
{
\begin{algorithm}
\caption{Online Learning over Neural Networks}\label{alg:ogd_nn}

\textbf{Input:} OCO algorithm $\A$, neural network $f(\cdot; \cdot)$, initial parameter $\theta_1$, parameter set $\Theta = B(R; \theta_1)$. \\
\For{$t=1\dots T$}{
Play $\theta_t$, receive loss  $\ell_t(\theta) = \ell_t(f(\theta; x_t))$. \\
Construct $h_t(\theta) = \ell_t(\theta_t) + \nabla \ell_t(\theta_t)^\top(\theta - \theta_t)$. \\
Update $\theta_{t+1} = \A(h_1, \dots, h_t) \in \Theta$.

}

\end{algorithm}

}

The main result in this work, provided in Theorem \ref{thm:main_nn}, gives a regret bound on the online agnostic learning of deep neural networks. The benchmark hypothesis class the online method competes against is a class of deep neural networks with interpolation dimension of at least $k$ where $k$ is known a priori and is used in the construction of the network.
\begin{theorem} \label{thm:main_nn} Let $\H_{\mathrm{NN}}(R; \theta_1) = \{ f(\theta; \cdot) : \theta \in \Theta \}$ denote the class of deep neural networks as in \eqref{eq:deep_nn} with parameter set $\Theta = B(R; \theta_1) = \{\theta: \|\theta[i] - \theta_1[i]\|_F \le R\ \forall\ i\in [d] \}$, and suppose Assumptions \ref{assumption:unit_norm} and \ref{assumption:loss} hold. Suppose $\gamma \in (0, O\left(\frac{1}{ H}\right)]$, take $R = O\left(\frac{k^3 \log m}{\gamma \sqrt{m}}\right)$, then for $m \ge O(\frac{p^{3/2}(k^{24}H^{12}\log^8 m + d)^{3/2}}{\gamma^8})$, with probability $1-O(H+d) e^{-\Omega(\log^2 m)}$ over the random initialization, 
\begin{itemize}
    \item The function class $\H_{\mathrm{NN}}(R; \theta_1)$ has interpolation dimension $\I_{\gamma}(\H_{\mathrm{NN}}(R; \theta_1)) \geq k$.
    \item Algorithm \ref{alg:ogd_nn} using OGD with $\eta_t = \frac{2R \sqrt{d}}{L H \sqrt{m}} \cdot t^{-1/2}$ for $\A$ attains regret bound
\begin{align*}
\sum_{t=1}^T \ell_t(f(\theta_t;x_t)) \le \min_{g \in \H_{\mathrm{NN}}(R; \theta_1)} \sum_{t=1}^T \ell_t(g(x_t)) + \tilde{O}\left( \frac{k^3 L H \sqrt{d T}}{\gamma} + \frac{k^4 L H^{5/2} \sqrt{d} T}{\gamma^{4/3}m^{1/6}} \right),
\end{align*}
where $\tilde{O}(\cdot)$ hides terms polylogarithmic in $m$.
\end{itemize}

\end{theorem}
Note that the radius $R$ of the permissible set for the parameters $B(R; \theta_1)$ scales inversely with $m$ indicating little movement in the parameter space, consistent with the recent insights in deep learning theory. 
Furthermore, the average regret can be minimized up to arbitrary precision: for any $\eps > 0$, if one chooses sufficiently large network width $m = \Omega(\eps^{-6})$ and sufficiently large number of iterations $T = \Omega(\eps^{-2})$, the average regret is bounded by $\eps$.

\section{Online learning with two-layer neural networks}\label{sec:online_twolayer}

To showcase the key ideas behind the main result in this work, we first consider a simpler setting as a warmup: online learning of two-layer neural networks. The setup in this section along with many of the derivations follow that of \citet{gao2019convergence}.
\paragraph{Two-layer Neural Networks.} For inputs $x\in \reals^{p}$, define the vector-valued two-layer neural network $f:\reals^{p} \rightarrow \reals^{d}$ with a smooth activation function $\sigma: \reals \rightarrow \reals$, even hidden layer width $m$ and weights $\theta \in \reals^{d \times m \times p}$ expressed as follows: for all $i \in [d]$ with parameter $\theta[i]\in \reals^{m\times p}, f(\theta; x) = (f_1(\theta[1]; x), \dots, f_d(\theta[d]; x))^\top \in \reals^d$, where
\begin{align}\label{eq:shallow_nn}
f_i(\theta[i]; x) = \frac{1}{\sqrt{m}}\big(\sum_{r=1}^{m/2} a_{i, r} \sigma(\theta[i,r]^\top x) + \sum_{r=1}^{m/2} \bar{a}_{i, r}\sigma(\bar{\theta}[i,r]^\top x)\big ).
\end{align}
The parameter for each $i \in [d]$ is given by $\theta[i] = (\theta[i, 1], \ldots, \theta[i, \frac{m}{2}], \bar{\theta}[i, 1], \ldots, \bar{\theta}[i, \frac{m}{2}])$. The scaling factor $\frac{1}{\sqrt{m}}$ is chosen optimally in retrospect of the analysis. We initialize $a_{i, r}$ to be randomly drawn from $\{\pm 1\}$, choose $\bar{a}_{i, r} = -a_{i, r}$, and fix them throughout training. The initialization scheme for $\theta$ is as follows: for all $i \in [d]$,  $\theta_1[i,r] \sim N(0, I_p)$ for $r = 1, \ldots, \frac{m}{2}$, and $\bar{\theta}_1[i, r] = \theta_1[i, r]$. This symmetric initialization scheme is chosen so that $f_i(\theta_1[i]; x) = 0$ for all $x\in \sphere_p$ to avoid some technical nuisance. We make the following assumption on the general activation function:
\begin{assumption} \label{assumption:activation}
The activation function $\sigma$ is $C-$Lipschitz and $C-$smooth: $|\sigma'(z)| \le C, \ |\sigma'(z) - \sigma'(z')| \le C|z - z'|$.
\end{assumption}

In this warmup setting, we quantify the expressivity of the hypothesis class via the Neural Tangent Kernel (NTK) \citep{jacot2018neural} of the network in \eqref{eq:shallow_nn}. Let $K_{\sigma}$ denote the NTK of the two-layer network and $\H_{\mathrm{RKHS}}(K_{\sigma})$ denote the RKHS of the $K_{\sigma}$ kernel. 
To obtain non-asymptotic guarantees, we restrict to RKHS functions of bounded norm. In this pursuit, we define the class of Random Feature functions $\H_{\mathrm{RF}}(\infty)$, which is \emph{dense} in $\H_{\mathrm{RKHS}}(K_{\sigma})$, construct its multidimensional analog $\H_{\mathrm{RF}}^d(\infty)$, and restrict it to the functions $\H^d_{\mathrm{RF}}(D)$ of bounded RF-norm $D$. See Appendix \ref{sec:app_warmup} for formal treatment. The regret bound given below consists of two parts: the regret for learning the optimal neural network parameters in the parameter set $\Theta$, and the approximation error of neural networks to the target function in $\H_{\mathrm{RF}}^d(D)$.

\begin{theorem}\label{thm:shallow_nn}
Let $f$ be a two-layer neural network as in \eqref{eq:shallow_nn} with the parameter set $\Theta = B(R; \theta_1) = \{ \theta \in \R^{d \times m \times p} \, : \, \| \theta - \theta_1 \|_F \le R \}$, and suppose Assumptions \ref{assumption:unit_norm}, \ref{assumption:loss}, \ref{assumption:activation} hold. For any $\delta>0, D > 1$, take $R= D\sqrt{d}$, then with probability at least $1-\delta$ over the random initialization, Algorithm \ref{alg:ogd_nn} with $\eta_t = \frac{2 R}{C L} \cdot t^{-1/2}$ satisfies
\begin{align*}
\sum_{t=1}^T \ell_t(f(\theta_t;x_t)) \le \min_{g \in \H_{\mathrm{RF}}^d(D)} \sum_{t=1}^T \ell_t(g(x_t)) &+ \tilde{O}\left(\frac{L\sqrt{dp}CD^2T}{\sqrt{m}}\right) + O\left(C L D\sqrt{dT}+ \frac{C L D^2 d T}{\sqrt{m}}\right)
\end{align*}
where $\tilde{O}(\cdot)$ hides factors that are polylogarithmic in $\delta, d$.
\end{theorem}
The radius $R$ being a constant w.r.t. $m$ indicates small movement of the parameters (number of parameters is linear in $m$). The regret bound is increasing in terms of $D$ which characterizes the expressivity of the benchmark function class $\H_{\mathrm{RF}}^d(D)$.
Finally, to achieve $\eps$ average regret, it suffices to take large enough width $m = \Omega(\eps^{-2})$ and large number of iterations $T = \Omega(\eps^{-2})$.


\paragraph{Proof structure.} The analysis of the above theorem goes through $3$ main components:
\begin{enumerate}
    \item \textbf{near convexity:} the losses $\ell_t(\theta) = \ell_t(f(\theta; x_t))$ are nearly convex (Lemma \ref{lem:shallow_almost_convex}).
    \item \textbf{regret guarantee:} regret is bounded against the parameter set $\Theta = B(R; \theta_1)$ (Lemma \ref{lem:shallow_convergence}).
    \item \textbf{expressivity:} any function $g \in \H_{\mathrm{RF}}^d(D)$ is approximated by a network (Lemma \ref{lem:shallow_approx}).
\end{enumerate}

\begin{lemma}\label{lem:shallow_almost_convex}
For any $\theta \in B(R; \theta_1)$ and any $t \in [T]$, the loss function $\ell_t(\theta)=\ell_t(f(\theta; x_t))$ is $\eps_{\text{nc}}$-nearly convex as in \eqref{eq:nearly_convex} with $\eps_{\text{nc}} = \frac{2CLR^2}{\sqrt{m}}$.
\end{lemma}

\begin{lemma} \label{lem:shallow_convergence} Algorithm \ref{alg:ogd_nn} with $\eta_t = \frac{2 R}{C L} \cdot t^{-1/2}$ attains regret bound 
\begin{align}\label{eq:ogd_nn_r}
\sum_{t=1}^T \ell_t(\theta_t) \le \min_{\theta \in B(R; \theta_1)} \sum_{t=1}^T \ell_t(\theta) + 3 C L R \cdot \sqrt{T} + \frac{2 C L R^2}{\sqrt{m}} \cdot T ~.
\end{align}
\end{lemma}


\begin{lemma}\label{lem:shallow_approx}
For any $\delta, D > 0$, let $g:\reals^p\rightarrow \reals^d \in \H_{\mathrm{RF}}^d(D)$, and let $R= D\sqrt{d}$, then with probability at least $1-\delta$ over the random initialization of $\theta_1$, there exists $\theta^* \in B(R; \theta_1)$ such that
\begin{align*}
 \forall x \in \sphere_p, \quad \ell_t(f(\theta^*;x)) \le \ell_t(g(x)) + \frac{L\sqrt{d}CD^2}{2\sqrt{m}} + \frac{L\sqrt{d}CD}{\sqrt{m}}(2\sqrt{2p} + 2\sqrt{\log d/\delta}) ~.
 \end{align*}
\end{lemma}

\section{Online Learning of Deep Neural Networks}\label{sec:dnn}
In this section, we extend the analysis sketched in Section \ref{sec:online_twolayer} to the setting of deep neural networks introduced in Section \ref{sec:dnn_def}. The extension follows the same $3$-step structure as outlined in Section \ref{sec:online_twolayer} albeit with more complex technical derivations. More importantly, the notion of expressivity in this setting is different: we do not provide approximation guarantees of the neural network class to an external function class (e.g. RKHS functions with bounded norm via NTK); rather we quantify the expressive power of the considered neural network class itself using the interpolation dimension (Definition \ref{def:interpolation_realdim}) as the metric. The proof structure for Theorem \ref{thm:main_nn} is analogous as provided below.
\begin{enumerate}
    \item \textbf{near convexity:} the losses $\ell_t(\theta) = \ell_t(f(\theta; x_t))$ are nearly convex (Lemma \ref{lem:deep_nc}).
    \item \textbf{regret guarantee:} regret is bounded against the parameter set $\Theta=B(R; \theta_1)$ (Lemma \ref{lem:deep_convergence}).
    \item \textbf{expressivity:} class of neural networks has interpolation dimension at least $k$ (Lemma \ref{lem:nn_interpol}).
\end{enumerate}

\begin{lemma}\label{lem:deep_nc}
Suppose $m \ge \Omega \left(\frac{p\log(1/R)+\log d}{R^{2/3}H} \right)$, and $R \le O\left(\frac{1}{H^6\log^3 m} \right)$, then with probability $1-O(H)e^{-\Omega(mR^{2/3}H)}$ over the random initialization, for any $\theta \in B(R; \theta_1)$, $x\in \mathbb{S}_p$, and $t \in [T]$, the loss function $\ell_t(\theta) = \ell_t(f(\theta; x))$ is $\eps_{\text{nc}}$-nearly convex with $\eps_{\text{nc}} = O\left( R^{4/3} L H^{5/2} \sqrt{d m \log m} \right)$.
\end{lemma}

\begin{lemma}\label{lem:deep_convergence}
Under conditions of Lemma \ref{lem:deep_nc}, Algorithm \ref{alg:ogd_nn} with $\eta_t = \frac{2 R \sqrt{d}}{L H \sqrt{m}} \cdot t^{-1/2}$ has regret
\begin{align}\label{eq:deep_regret_param}
\sum_{t=1}^T \ell_t(\theta_t) \le \min_{\theta \in B(R; \theta_1)} \sum_{t=1}^T \ell_t(\theta) + O\left(R L H \sqrt{dmT}\right) + O\left( R^{4/3} L H^{5/2} \sqrt{d m \log m} T \right) ~.
\end{align}
\end{lemma}
\section{Online Episodic Control with Neural Network Controllers}\label{sec:control}
\subsection{Online Episodic Control}\label{sec:prelim_control}
Consider the following online episodic learning problem for nonstochastic control over linear time-varying (LTV) dynamics: there is a sequence of $T$ control problems each with a horizon $K$ and an initial state $x_1 \in \reals^{d_x}$. In each episode, the state transition is given by
\begin{equation}\label{eq:dynamics}
    \forall k \in [1, K], \quad x_{k+1} = A_k x_k + B_k u_k + w_k,
\end{equation}
where $x_k \in \R^{d_x}, u_k \in \R^{d_u}$. The system matrices $A_k \in \R^{d_x \times d_x}, B_k \in \R^{d_x \times d_u}$ along with the next state $x_{k+1}$ are revealed to the learner \emph{after} taking the action $u_k$. The disturbances $w_k \in \R^{d_x}$ are unknown and adversarial but can be a posteriori computed by the learner $w_k = x_{k+1} - A_k x_k - B_k u_k$. An episode loss is defined cumulatively over the rounds $k \in [1, K]$ according to the cost functions $c_k : \R^{d_x} \times \R^{d_u} \to \R$ of state and action: for a policy $\pi$, the loss is $J (\pi; x_1, c_{1:K}) =  \sum_{k=1}^{K} c_k(x_k^\pi, u_k^\pi)$.
The transition matrices $(A_k, B_k)_{1:K}$, initial state $x_1$, disturbances $w_{1:K}$ and costs $c_{1:K}$ can change arbitrarily for different episodes. The goal of the learner is to minimize {\it episodic} regret by adapting its output policies $\pi_t$ for $t \in [1, T]$,
\begin{equation}\label{eq:control_regret}
    \regret_T(\Pi) = \sum_{t=1}^T J_t(\pi_t; x^t_1, c^t_{1:K}) - \min_{\pi \in \Pi}  \sum_{t=1}^T J_t(\pi; x^t_1, c^t_{1:K}),
\end{equation}
where $\Pi$ denotes the class of policies the learner competes against.

The model above is presented in its utmost generality: the system in an episode is LTV and these LTVs are allowed to change arbitrarily throughout episodes. Results for this model can be applied to derive guarantees for: (1) a simpler setting, learning to control a single LTV episodically; (2) a more complex setting, first-order trajectory-based control or planning over \emph{nonlinear} dynamics by taking the Jacobian linearization of the dynamics \citep{ahn2007ilc, westenbroek2021stability}. We make the following basic assumptions about the dynamical system {\it in each episode} that are common in the nonstochastic control literature \citep{agarwal2019online}.
\begin{assumption}\label{assumption:control_bound}
All disturbances have a uniform bound on their norms: $\max_{k \in [K]} \| w_k \|_2 \le W$.
\end{assumption}
\begin{assumption}\label{assumption:control_stable}
There exist $C_1, C_2 \ge 1$ and $0<\rho_1<1$ such that the system matrices satisfy:
\begin{align*}
    \forall k \in [K], \forall n \in [1, k), \quad \left\| \prod_{i=k}^{k-n+1} A_i \right\|_{\text{op}} \leq C_1 \cdot \rho_1^n, \quad \| B_k \|_{\text{op}} \leq C_2 ~.
\end{align*}
\end{assumption}
\begin{assumption}\label{assumption:control_loss}
Each cost function $c_k : \R^{d_x} \times \R^{d_u} \to \R$ is jointly convex and satisfies a generalized Lipschitz condition $\| \nabla c_k(x, u) \| \leq L_c \max\{1, \|x\|+\|u\| \}$ for some $L_c > 0$.
\end{assumption}

\subsection{Neural networks as the policy class}\label{sec:control_policy}
The performance of the learner given by \eqref{eq:control_regret} directly depends on the policy class $\Pi$. 
In this work, we focus on disturbance based policies, i.e. policies that take past perturbations as input $u_k = f(w_{1:k-1})$. These policies are parameterized w.r.t. {\it policy-independent} inputs, in this case the sequence $w_{1:K}$.  This is in contrast to the commonly used state feedback policy $u_k = f( x_k)$.  
For example, the DAC policy class \citep{agarwal2019online} outputs controls linear in past finite disturbances resulting in a \emph{convex} parameterization of the state and control and enabling the design of efficient provable online methods
We expand the comparator class by considering policies with controls that are \emph{nonlinear} in the past disturbances, represented by a neural network.
\begin{definition} (Disturbance Neural Feedback Control) A disturbance neural feedback policy $\pi^{\theta}_{\text{dnn}}$ chooses control $u_k$ output by a neural network over the past disturbances,
\begin{align*}
u_k = f_{\theta}(w_{k-1}, w_{k-2}, \ldots, w_{1}) \in \R^{d_u},
\end{align*}
where $f_{\theta}(\cdot)$ is a neural network with parameters $\theta$.
\end{definition}
The reasoning behind this policy class expansion is twofold. First, even for LTI systems subject to adversarial disturbances $w_k$ and general costs $c_k$, the best-in-hindsight linear policies (DAC, DRC, state feedback) are not necessarily close to the optimal open-loop control sequence. Furthermore, as already mentioned, our episodic LTV setting can be used for first-order guarantees over \emph{nonlinear} dynamics. Hence, competing against the rich policy class of neural network controllers is highly desirable.
For a given neural network architecture let $f_{\theta}(\cdot) = f(\theta; \cdot)$ defined by \eqref{eq:deep_nn}, and let $\Theta$ be the parameter set. The class of deep controller policies is denoted as $\Pi_{\text{dnn}}(f; \Theta) = \{ \pi_{\text{dnn}}^{\theta} : \theta \in \Theta \}$.

\paragraph{Benchmark Policy Class Capacity.} The online episodic control problem described in Section \ref{sec:prelim_control} with the policy class $\Pi = \Pi_{\text{dnn}}(f;\Theta)$ can be reduced to online learning for neural networks (Section \ref{sec:olnn}). For simplicity, we temporarily drop the index $t \in [T]$ of a single episode and define the optimal \emph{open-loop} control sequence of an episode.
\begin{definition}\label{def:open_loop_optimal}
Define the optimal open-loop control sequence $u^*_{1:K} \in [-1, 1]^{K \times d_u}$ to be
\begin{align*}
    u^*_{1:K} = \argmin_{\forall k, u_k \in [-1, 1]^{d_u}} \left\{ J(u_{1:K}; x_1, c_{1:K}) = \sum_{k=1}^K c_k(x_k, u_k) \right\} ~.
\end{align*}
\end{definition}
For each $k \in [K]$, denote the padded input $z_k = \text{vec}([w_{k-1}, \dots, w_1, \mathbf{0}, \dots, \mathbf{0}, k]) \in \R^{K \cdot d_x + 1}$. To satisfy Assumption \ref{assumption:unit_norm}, normalize the network inputs $\bar{z}_k = \frac{z_k}{\|z_k\|_2} \in \sphere_{K \cdot d_x + 1}$. We demonstrate the capacity of the benchmark policy class $\Pi_{\text{dnn}}(f; \Theta)$ with $\Theta = B(R; \theta_1)$ by showing that it can output the \emph{optimal} open-loop control sequence of any \emph{single} episode up to arbitrary precision.

\begin{lemma}\label{lem:control_openloop}
Take $R = O\left( \frac{K^3 \log m (2KW+H)}{\sqrt{m}} \right)$, suppose $m \geq \Omega\left(K^{24} H^{12} \log^5 m (2KW+H)^8\right)$, then with probability $1 - d_u \cdot e^{-\Omega(\log^2 m)}$ over the random initialization of $\theta_1$, $\Pi_{\text{dnn}}(f; \Theta)$ can output any open-loop control sequence $u_{1:K}^* \in [-1, 1]^{K \times d_u}$ up to arbitrary precision:
\begin{align*}
    \forall \eps > 0, \ \exists \pi_{\text{dnn}}^{\theta} \in \Pi_{\text{dnn}}(f; \Theta), \text{ s.t. } \sum_{k=1}^K \| u_k^{\theta} - u_k^* \|_2^2 \leq \eps ~.
\end{align*}
\end{lemma}

\paragraph{Episodic regret bounds.}
For a policy $\pi_{\text{dnn}}^{\theta}$ the episode loss $\mL(\theta) = J(\pi_{\text{dnn}}^{\theta}; x_1, c_{1:K})$ depends on the parameter $\theta$ through all the $K$ controls $u_k^{\theta} = f(\theta; \bar{z}_k)$. Denote $\bar{f}(\theta) = [u_1^{\theta}, \dots, u_K^{\theta}]^{\top} \in \R^{K \times d_u}$ and let $\mL(\theta) = \mL(\bar{f}(\theta))$ by abuse of notation. We demonstrate that the reduction to the online learning setting is achieved by showing that $\mL(\bar{f}(\theta))$ satisfies the convexity (Lemma \ref{lem:control_convex}) and Lipschitz (Lemma \ref{lem:control_lip}) conditions. Hence, for each episode $t\in [T]$, the episode loss $\mL_t(\theta) = J_t(\pi_{\text{dnn}}^{\theta}; x_1^t, c_{1:K}^t)$ satisfies Assumption \ref{assumption:loss} and the rest of the derivation is analogous to that of Section \ref{sec:olnn} (see Appendix \ref{sec:app_control}). The algorithm itself for online episodic control is simply projected OGD (can be any OCO algorithm) over the losses $\mL_t(\theta)$ given in detail in Algorithm \ref{alg:control}.

\iftoggle{arxiv}
{
\begin{algorithm}
\caption{Deep Neural Network Episodic Control with OGD}
\label{alg:control}
\begin{algorithmic}[1]
\STATE \textbf{Input:} $\eta_t > 0$, initial parameter $\theta_1$, permissible set $\Theta$.
\FOR{$t=1\dots T$}
\FOR{$k=1\dots K$}
\STATE Observe $x_{k}^t$ and play $u_{k}^t = f(\theta_t, \bar{z}^t_k)$.
\ENDFOR
\STATE Construct loss function 
$$\mL_t(\theta) = \sum_{k=1}^{K} c^t_k(x^{t, \theta}_{k}, f(\theta, \bar{z}_k^t).
$$
\STATE Perform gradient update $\theta_{t+1} = \Pi_{\Theta}(\theta_t - \eta_t \nabla_{\theta} \mL_t(\theta_t)).$
\ENDFOR
\end{algorithmic}
\end{algorithm}
}
{
\begin{algorithm}
\caption{Online Episodic Control with Neural Network Policies}
\label{alg:control}
\textbf{Input:} OCO algorithm $\A$, neural network $f(\cdot; \cdot)$, initial parameter $\theta_1$, parameter set $\Theta$. \\
\For{$t=1\dots T$}{
\For{$k=1\dots K$}{
Observe $x_{k}^t$, play $u_{k}^t = f(\theta_t, \bar{z}^t_k)$, observe $A_k^t, B_k^t$.
}
Construct loss function $\mL_t(\theta) = \sum_{k=1}^{K} c^t_k(x^{t, \theta}_{k}, f(\theta, \bar{z}_k^t)$. \\
Perform gradient update $\theta_{t+1} = \Pi_{\Theta}(\theta_t - \eta_t \nabla_{\theta} \mL_t(\theta_t)).$
}

\end{algorithm}
}

\begin{theorem}\label{thm:control}
Let $\Pi_{\text{dnn}}(f; \Theta)$ denote the class of policies $\pi_{\text{dnn}}^{\theta} = f(\theta; \cdot)$ which is defined by \eqref{eq:deep_nn} with $\theta \in \Theta = B(R; \theta_1)$ and suppose Assumptions \ref{assumption:control_bound}, \ref{assumption:control_stable}, \ref{assumption:control_loss} hold. Take $R = O\left(\frac{K^3 (2KW+H) \log m}{\sqrt{m}}\right)$, then for $m \ge \Omega(K^{46} H^{20} W^8 (d_x d_u)^{3/2} \log^{12} m)$ with probability at least $1-O(H+d_u) e^{-\Omega(\log^2 m)}$ over the randomness of initialization $\theta_1$,
Algorithm \ref{alg:control} with $\eta_t = O(\frac{R \sqrt{d_u}}{L H \sqrt{m}}t^{-1/2})$, attains episodic regret bound given by
\begin{align*}
    \regret_T(\Pi_{\text{dnn}}(f; \Theta)) \le \tilde{O}\left(K^{10} L_c H^4 W^2 d_u d_x^{1/2} \cdot \sqrt{T} + \frac{K^{12} L_c H^6 W^3 d_u d_x^{1/2}}{m^{1/6}} \cdot T \right),
\end{align*}
where $\Pi_{\text{dnn}}(f; \Theta)$ can output any open-loop control sequence $u_{1:K}^* \in [-1, 1]^{K \times d_u}$ up to arbitrary precision and $\tilde{O}(\cdot)$ hides terms polylogarithmic in $m$.
\end{theorem}
This theorem statement, analogous to Theorem \ref{thm:main_nn}, implies that arbitrarily small $\eps>0$ average regret is attained with a large network width $m = \Omega(\eps^{-6})$ and large number of iterations $T = \Omega(\eps^{-2})$.




\section{Conclusion and Future Work}

Deep neural networks are nonconvex predictors and do not readily fit into known {\it efficient} online learning frameworks. We describe a reduction that takes any OCO algorithm and converts it into an online deep learning algorithm with provable regret bounds against the best net in hindsight. Moreover, through the choice of the parameter set, we can ensure that the class of deep neural networks has provable expressivity, as measured according to the interpolation dimension. 

This black box reduction disentangles features of deep learning that were considered to be inherently related: expressivity, optimization method, and generalization. 
In addition, we give an application where agnostic online learning is crucial: online episodic control. For this application we derive the first provable performance guarantees for neural network based controllers. 

Numerous exciting open questions arise: can we prove tighter linear interpolation dimension bounds for neural networks? what is the optimal trade-off between expressivity and regret? What are the tight regret bounds for an $m$-parameter neural network?

\section*{Acknowledgements}

The authors thank Shay Moran and Samory Kpotufe for helpful discussions on the interpolation dimension. 
Xinyi Chen, Edgar Minasyan and Elad Hazan acknowledge funding from NSF awards 2134040 and 1704860. All authors acknowledge funding from the Google corporation. JDL acknowledges support of the ARO under MURI Award W911NF-11-1-0304,  the Sloan Research Fellowship, NSF CCF 2002272, NSF IIS 2107304,  and an ONR Young Investigator Award. 
\newpage
\bibliography{references}

\begin{thebibliography}{44}
\providecommand{\natexlab}[1]{#1}
\providecommand{\url}[1]{\texttt{#1}}
\expandafter\ifx\csname urlstyle\endcsname\relax
  \providecommand{\doi}[1]{doi: #1}\else
  \providecommand{\doi}{doi: \begingroup \urlstyle{rm}\Url}\fi

\bibitem[Abbasi-Yadkori and Szepesv{\'a}ri(2011)]{abbasi2011regret}
Yasin Abbasi-Yadkori and Csaba Szepesv{\'a}ri.
\newblock Regret bounds for the adaptive control of linear quadratic systems.
\newblock In \emph{Proceedings of the 24th Annual Conference on Learning
  Theory}, pages 1--26, 2011.

\bibitem[Agarwal et~al.(2019)Agarwal, Bullins, Hazan, Kakade, and
  Singh]{agarwal2019online}
Naman Agarwal, Brian Bullins, Elad Hazan, Sham Kakade, and Karan Singh.
\newblock Online control with adversarial disturbances.
\newblock In \emph{International Conference on Machine Learning}, pages
  111--119, 2019.

\bibitem[Ahn et~al.(2007)Ahn, Chen, and Moore]{ahn2007ilc}
Hyo-Sung Ahn, YangQuan Chen, and Kevin~L. Moore.
\newblock Iterative learning control: Brief survey and categorization.
\newblock \emph{IEEE Transactions on Systems, Man, and Cybernetics, Part C
  (Applications and Reviews)}, 37\penalty0 (6):\penalty0 1099--1121, 2007.
\newblock \doi{10.1109/TSMCC.2007.905759}.

\bibitem[Allen-Zhu et~al.(2019)Allen-Zhu, Li, and
  Song]{allenzhu2019convergence}
Zeyuan Allen-Zhu, Yuanzhi Li, and Zhao Song.
\newblock A convergence theory for deep learning via over-parameterization,
  2019.

\bibitem[Arora et~al.(2021)Arora, Arora, Bruna, Cohen, Du, Ge, Gunasekar, Jin,
  Lee, Ma, and Neyshabur]{arorabook}
Raman Arora, Sanjeev Arora, Joan Bruna, Nadav Cohen, Simon Du, Rong Ge, Suriya
  Gunasekar, Chi Jin, Jason Lee, Tengyu Ma, and Behnam Neyshabur.
\newblock \emph{Theory of Deep Learning}.
\newblock 2021.

\bibitem[Arora et~al.(2019)Arora, Du, Hu, Li, and Wang]{arora2019fine}
Sanjeev Arora, Simon~S Du, Wei Hu, Zhiyuan Li, and Ruosong Wang.
\newblock Fine-grained analysis of optimization and generalization for
  overparameterized two-layer neural networks.
\newblock \emph{arXiv preprint arXiv:1901.08584}, 2019.

\bibitem[Bai and Lee(2019)]{bai2019beyond}
Yu~Bai and Jason~D Lee.
\newblock Beyond linearization: On quadratic and higher-order approximation of
  wide neural networks.
\newblock \emph{arXiv preprint arXiv:1910.01619}, 2019.

\bibitem[Blondel and Tsitsiklis(2000)]{blondel2000survey}
Vincent~D Blondel and John~N Tsitsiklis.
\newblock A survey of computational complexity results in systems and control.
\newblock \emph{Automatica}, 36\penalty0 (9):\penalty0 1249--1274, 2000.

\bibitem[Bubeck and Sellke(2021)]{bubeck2021a}
Sebastien Bubeck and Mark Sellke.
\newblock A universal law of robustness via isoperimetry.
\newblock In A.~Beygelzimer, Y.~Dauphin, P.~Liang, and J.~Wortman Vaughan,
  editors, \emph{Advances in Neural Information Processing Systems}, 2021.
\newblock URL \url{https://openreview.net/forum?id=z71OSKqTFh7}.

\bibitem[Cai et~al.(2019)Cai, Gao, Hou, Chen, Wang, He, Zhang, and
  Wang]{cai2019gram}
Tianle Cai, Ruiqi Gao, Jikai Hou, Siyu Chen, Dong Wang, Di~He, Zhihua Zhang,
  and Liwei Wang.
\newblock Gram-gauss-newton method: Learning overparameterized neural networks
  for regression problems.
\newblock \emph{arXiv preprint arXiv:1905.11675}, 2019.

\bibitem[Cao and Gu(2019)]{cao2019generalization}
Yuan Cao and Quanquan Gu.
\newblock Generalization bounds of stochastic gradient descent for wide and
  deep neural networks.
\newblock \emph{Advances in Neural Information Processing Systems},
  32:\penalty0 10836--10846, 2019.

\bibitem[Cesa-Bianchi and Lugosi(2006)]{cesa2006prediction}
Nicolo Cesa-Bianchi and G{\'a}bor Lugosi.
\newblock \emph{Prediction, learning, and games}.
\newblock Cambridge university press, 2006.

\bibitem[Cohen et~al.(2019)Cohen, Koren, and Mansour]{cohen2019learning}
Alon Cohen, Tomer Koren, and Yishay Mansour.
\newblock Learning linear-quadratic regulators efficiently with only $\sqrt{T}$
  regret.
\newblock In \emph{International Conference on Machine Learning}, pages
  1300--1309, 2019.

\bibitem[Dean et~al.(2018)Dean, Mania, Matni, Recht, and Tu]{dean2018regret}
Sarah Dean, Horia Mania, Nikolai Matni, Benjamin Recht, and Stephen Tu.
\newblock Regret bounds for robust adaptive control of the linear quadratic
  regulator.
\newblock In \emph{Advances in Neural Information Processing Systems}, pages
  4188--4197, 2018.

\bibitem[Du et~al.(2018{\natexlab{a}})Du, Lee, Li, Wang, and
  Zhai]{du2018gradient_deep}
Simon~S Du, Jason~D Lee, Haochuan Li, Liwei Wang, and Xiyu Zhai.
\newblock Gradient descent finds global minima of deep neural networks.
\newblock \emph{arXiv preprint arXiv:1811.03804}, 2018{\natexlab{a}}.

\bibitem[Du et~al.(2018{\natexlab{b}})Du, Zhai, Poczos, and
  Singh]{du2018gradient}
Simon~S Du, Xiyu Zhai, Barnabas Poczos, and Aarti Singh.
\newblock Gradient descent provably optimizes over-parameterized neural
  networks.
\newblock \emph{arXiv preprint arXiv:1810.02054}, 2018{\natexlab{b}}.

\bibitem[Gao et~al.(2019)Gao, Cai, Li, Wang, Hsieh, and
  Lee]{gao2019convergence}
Ruiqi Gao, Tianle Cai, Haochuan Li, Liwei Wang, Cho-Jui Hsieh, and Jason~D.
  Lee.
\newblock Convergence of adversarial training in overparametrized neural
  networks, 2019.

\bibitem[Hazan(2019)]{hazan2019introduction}
Elad Hazan.
\newblock Introduction to online convex optimization.
\newblock \emph{arXiv preprint arXiv:1909.05207}, 2019.

\bibitem[Hazan and Singh(2021)]{hazan2021tutorial}
Elad Hazan and Karan Singh.
\newblock Tutorial: online and non-stochastic control, July 2021.

\bibitem[Jacot et~al.(2018)Jacot, Gabriel, and Hongler]{jacot2018neural}
Arthur Jacot, Franck Gabriel, and Cl{\'e}ment Hongler.
\newblock Neural tangent kernel: Convergence and generalization in neural
  networks.
\newblock \emph{arXiv preprint arXiv:1806.07572}, 2018.

\bibitem[Ji and Telgarsky(2020)]{ji2020polylogarithmic}
Ziwei Ji and Matus Telgarsky.
\newblock Polylogarithmic width suffices for gradient descent to achieve
  arbitrarily small test error with shallow relu networks, 2020.

\bibitem[Kakade et~al.(2020)Kakade, Krishnamurthy, Lowrey, Ohnishi, and
  Sun]{kakade2020information}
Sham Kakade, Akshay Krishnamurthy, Kendall Lowrey, Motoya Ohnishi, and Wen Sun.
\newblock Information theoretic regret bounds for online nonlinear control.
\newblock In H.~Larochelle, M.~Ranzato, R.~Hadsell, M.~F. Balcan, and H.~Lin,
  editors, \emph{Advances in Neural Information Processing Systems}, volume~33,
  pages 15312--15325. Curran Associates, Inc., 2020.
\newblock URL
  \url{https://proceedings.neurips.cc/paper/2020/file/aee5620fa0432e528275b8668581d9a8-Paper.pdf}.

\bibitem[Levine et~al.(2016)Levine, Finn, Darrell, and Abbeel]{levine2016end}
Sergey Levine, Chelsea Finn, Trevor Darrell, and Pieter Abbeel.
\newblock End-to-end training of deep visuomotor policies.
\newblock \emph{The Journal of Machine Learning Research}, 17\penalty0
  (1):\penalty0 1334--1373, 2016.

\bibitem[Lewis et~al.(1997)Lewis, Jagannathan, and Ye{\c
  s}ildirek]{LEWIS1997161}
F.L. Lewis, S.~Jagannathan, and A.~Ye{\c s}ildirek.
\newblock Chapter 7 - neural network control of robot arms and nonlinear
  systems.
\newblock In Omid Omidvar and David~L. Elliott, editors, \emph{Neural Systems
  for Control}, pages 161--211. Academic Press, San Diego, 1997.
\newblock ISBN 978-0-12-526430-3.
\newblock \doi{https://doi.org/10.1016/B978-012526430-3/50008-8}.

\bibitem[Lillicrap et~al.(2016)Lillicrap, Hunt, Pritzel, Heess, Erez, Tassa,
  Silver, and Wierstra]{lillicrap2016continuous}
Timothy~P Lillicrap, Jonathan~J Hunt, Alexander Pritzel, Nicolas Heess, Tom
  Erez, Yuval Tassa, David Silver, and Daan Wierstra.
\newblock Continuous control with deep reinforcement learning.
\newblock In \emph{ICLR (Poster)}, 2016.

\bibitem[Malach et~al.(2021)Malach, Yehudai, Shalev-Schwartz, and
  Shamir]{malach2021connection}
Eran Malach, Gilad Yehudai, Shai Shalev-Schwartz, and Ohad Shamir.
\newblock The connection between approximation, depth separation and
  learnability in neural networks.
\newblock In Mikhail Belkin and Samory Kpotufe, editors, \emph{Proceedings of
  Thirty Fourth Conference on Learning Theory}, volume 134 of \emph{Proceedings
  of Machine Learning Research}, pages 3265--3295. PMLR, 15--19 Aug 2021.
\newblock URL \url{https://proceedings.mlr.press/v134/malach21a.html}.

\bibitem[Mania et~al.(2019)Mania, Tu, and Recht]{mania2019certainty}
Horia Mania, Stephen Tu, and Benjamin Recht.
\newblock Certainty equivalence is efficient for linear quadratic control.
\newblock In \emph{Advances in Neural Information Processing Systems}, pages
  10154--10164, 2019.

\bibitem[Mohri et~al.(2018)Mohri, Rostamizadeh, and
  Talwalkar]{mohri2018foundations}
Mehryar Mohri, Afshin Rostamizadeh, and Ameet Talwalkar.
\newblock \emph{Foundations of machine learning}.
\newblock MIT press, 2018.

\bibitem[Moore(2012)]{moore2012iterative}
Kevin~L Moore.
\newblock \emph{Iterative learning control for deterministic systems}.
\newblock Springer Science \& Business Media, 2012.

\bibitem[Rahimi and Recht(2008)]{rahimi2008uniform}
Ali Rahimi and Benjamin Recht.
\newblock Uniform approximation of functions with random bases.
\newblock In \emph{2008 46th Annual Allerton Conference on Communication,
  Control, and Computing}, pages 555--561, 2008.
\newblock \doi{10.1109/ALLERTON.2008.4797607}.

\bibitem[Rajeswaran et~al.(2017)Rajeswaran, Lowrey, Todorov, and
  Kakade]{rajeswaran2017towards}
Aravind Rajeswaran, Kendall Lowrey, Emanuel Todorov, and Sham Kakade.
\newblock Towards generalization and simplicity in continuous control.
\newblock \emph{arXiv preprint arXiv:1703.02660}, 2017.

\bibitem[Shalev-Shwartz and Ben-David(2014)]{shalev2014understanding}
Shai Shalev-Shwartz and Shai Ben-David.
\newblock \emph{Understanding machine learning: From theory to algorithms}.
\newblock Cambridge university press, 2014.

\bibitem[Simchowitz and Foster(2020)]{simchowitz2020naive}
Max Simchowitz and Dylan Foster.
\newblock Naive exploration is optimal for online lqr.
\newblock In \emph{International Conference on Machine Learning}, pages
  8937--8948. PMLR, 2020.

\bibitem[Soltanolkotabi et~al.(2018)Soltanolkotabi, Javanmard, and
  Lee]{soltanolkotabi2018theoretical}
Mahdi Soltanolkotabi, Adel Javanmard, and Jason~D Lee.
\newblock Theoretical insights into the optimization landscape of
  over-parameterized shallow neural networks.
\newblock \emph{IEEE Transactions on Information Theory}, 2018.

\bibitem[{Tassa} et~al.(2012){Tassa}, {Erez}, and {Todorov}]{iLQR}
Y.~{Tassa}, T.~{Erez}, and E.~{Todorov}.
\newblock Synthesis and stabilization of complex behaviors through online
  trajectory optimization.
\newblock In \emph{2012 IEEE/RSJ International Conference on Intelligent Robots
  and Systems}, pages 4906--4913, 2012.

\bibitem[Todorov and Li(2005)]{todorov2005generalized}
Emanuel Todorov and Weiwei Li.
\newblock A generalized iterative lqg method for locally-optimal feedback
  control of constrained nonlinear stochastic systems.
\newblock In \emph{Proceedings of the 2005, American Control Conference,
  2005.}, pages 300--306. IEEE, 2005.

\bibitem[Vapnik(1999)]{vapnik1999nature}
Vladimir Vapnik.
\newblock \emph{The nature of statistical learning theory}.
\newblock Springer science \& business media, 1999.

\bibitem[Wei et~al.(2019)Wei, Lee, Liu, and Ma]{wei2019regularization}
Colin Wei, Jason Lee, Qiang Liu, and Tengyu Ma.
\newblock Regularization matters: Generalization and optimization of neural
  nets vs their induced kernel.
\newblock 2019.

\bibitem[Westenbroek et~al.(2021)Westenbroek, Simchowitz, Jordan, and
  Sastry]{westenbroek2021stability}
Tyler Westenbroek, Max Simchowitz, Michael~I Jordan, and S~Shankar Sastry.
\newblock On the stability of nonlinear receding horizon control: a geometric
  perspective.
\newblock \emph{arXiv preprint arXiv:2103.15010}, 2021.

\bibitem[Wu et~al.(2019)Wu, Du, and Ward]{wu2019global}
Xiaoxia Wu, Simon~S Du, and Rachel Ward.
\newblock Global convergence of adaptive gradient methods for an
  over-parameterized neural network.
\newblock \emph{arXiv preprint arXiv:1902.07111}, 2019.

\bibitem[Wu et~al.(2021)Wu, Xie, Du, and Ward]{wu2021adaloss}
Xiaoxia Wu, Yuege Xie, Simon Du, and Rachel Ward.
\newblock Adaloss: A computationally-efficient and provably convergent adaptive
  gradient method.
\newblock \emph{arXiv preprint arXiv:2109.08282}, 2021.

\bibitem[Yehudai and Shamir(2019)]{yehudai2019power}
Gilad Yehudai and Ohad Shamir.
\newblock On the power and limitations of random features for understanding
  neural networks.
\newblock In H.~Wallach, H.~Larochelle, A.~Beygelzimer, F.~d\textquotesingle
  Alch\'{e}-Buc, E.~Fox, and R.~Garnett, editors, \emph{Advances in Neural
  Information Processing Systems}, volume~32. Curran Associates, Inc., 2019.
\newblock URL
  \url{https://proceedings.neurips.cc/paper/2019/file/5481b2f34a74e427a2818014b8e103b0-Paper.pdf}.

\bibitem[Zhang et~al.(2020)Zhang, Plevrakis, Du, Li, Song, and
  Arora]{zhang2020overparameterized}
Yi~Zhang, Orestis Plevrakis, Simon~S. Du, Xingguo Li, Zhao Song, and Sanjeev
  Arora.
\newblock Over-parameterized adversarial training: An analysis overcoming the
  curse of dimensionality, 2020.

\bibitem[Zou et~al.(2018)Zou, Cao, Zhou, and Gu]{zou2018stochastic}
Difan Zou, Yuan Cao, Dongruo Zhou, and Quanquan Gu.
\newblock Stochastic gradient descent optimizes over-parameterized deep {ReLU}
  networks.
\newblock \emph{arXiv preprint arXiv:1811.08888}, 2018.

\end{thebibliography}
\bibliographystyle{plainnat}
\newpage
\appendix
\tableofcontents
\newpage
\section{Details for Section \ref{sec:prelim}}
\subsection{Interpolation dimension}\label{sec:app_interpol}
\begin{proof}[Proof of Lemma \ref{lem:nn_interpol}]
Let $\{(x_j, y_j)\}_{j=1}^k$ be a set of examples where $x_j\in \mathbb{S}_p$, $y_j\in [-1, 1]^d$, and the $x_j$'s are at least $\gamma$ apart, i.e. $\min_{j \neq l} \| x_j - x_l \|_2 \geq \gamma$ with $\gamma\in (0, O\left(\frac{1}{H}\right)]$. Let $y_{j, i}$ denote the $i$-th coordinate of the label $y_j$, and recall that $f_i(\theta[i];x)$ is the scalar output of the vector-valued deep neural network at coordinate $i$, with parameters $\theta[i]$ and input $x$. Fix any arbitrary $\eps > 0$. By Theorem 1 in \citet{allenzhu2019convergence}, for $m\ge \Omega(\frac{k^{24}H^{12}\log^5 m}{\gamma^8})$, $R = O(\frac{k^3\log m}{\gamma\sqrt{m}})$, for any fixed $i\in [d]$, with probability at least $1-e^{-\Omega(\log^2 m)}$, there exists $\theta^*[i]$ such that $\|\theta^*[i] - \theta_1[i]\|_F \le R$, and
$$
\sum_{j=1}^k(f_i(\theta^*[i]; x_j) - y_{j, i})^2 \le \frac{\eps}{d}.
$$
The existence of such $\theta^*$ follows from the statement of the aforementioned theorem, i.e. gradient descent finds such $\theta^*$ in a finite number of iterations (convergence rate is irrelevant). Note that our choice of $m$ and $R$ satisfy the above conditions. Taking a union bound, we conclude that with probability at least $1-d \cdot e^{-\Omega(\log^2 m)}$, there exists $\theta^* = (\theta^*[1], \ldots, \theta^*[d])$ such that for all $i$, $\|\theta^*[i] - \theta_1[i]\|_F\le R$, and
$$
 \sum_{j=1}^k \|f(\theta^*;x_j) - y_{j}\|_2^2 = \sum_{i=1}^d \sum_{j=1}^k (f_i(\theta^*[i];x_j) - y_{j, i})^2 \le \eps.
$$
This conclusion is true for any $\eps>0$ and training set $\{(x_j, y_j)\}_{j=1}^k$ satisfying the stated conditions. Thus, the interpolation dimension of $\H_{\mathrm{NN}}(R)$, at non-degeneracy $\gamma$, is lower bounded by $k$.
\end{proof}
\subsection{Online nearly convex optimization}\label{sec:app_onco}
The full algorithm for extending OCO to nearly convex loss functions $\ell_t$ is presented in Algorithm \ref{alg:onco}. In addition to the proof of Lemma \ref{lem:ogd_almost_convex}, we provide a corollary with OGD as the OCO algorithm $\A$ to use the explicit regret bound in further derivations. The proof of the corollary simply follows by plugging in the appropriate regret (and stepsize) value for OGD.
\iftoggle{arxiv}
{
\begin{algorithm}
\caption{Online Nearly-Convex Optimization} \label{alg:onco}
\begin{algorithmic}[1]
\STATE \textbf{Input:} OCO algorithm $\A$ for convex decision set $\K$.
\FOR{$t=1\dots T$}
\STATE Play $\theta_t$, observe $\ell_t$.
\STATE Construct $h_t(\theta) = \ell_t(\theta_t) + \nabla \ell_t(\theta_t)^\top (\theta - \theta_t) $.
\STATE Update $\theta_{t+1} = \A(h_1,...,h_t) \in \K$.
\ENDFOR
\end{algorithmic}
\end{algorithm}
}
{
\begin{algorithm}
\caption{Online Nearly-Convex Optimization} \label{alg:onco}

\textbf{Input:} OCO algorithm $\A$ for convex decision set $\K$. \\
\For{$t=1\dots T$}{
Play $\theta_t$, observe $\ell_t$. \\
Construct $h_t(\theta) = \ell_t(\theta_t) + \nabla \ell_t(\theta_t)^\top (\theta - \theta_t) $. \\
Update $\theta_{t+1} = \A(h_1,...,h_t) \in \K.$
}

\end{algorithm}
}
\begin{proof}[Proof of Lemma \ref{lem:ogd_almost_convex}]
Observe that by the $\eps$-nearly convex property, for all $\theta \in \K$,
$$
h_t(\theta) - \ell_t(\theta) = \ell_t(\theta_t) + \nabla \ell_t(\theta_t)^\top (\theta - \theta_t) - \ell_t(\theta) \le \eps.
$$
Moreover, by construction the functions $h_t(\cdot)$ are convex and $h_t(\theta_t) = \ell_t(\theta_t)$ for all $t\in [T]$. The regret can be decomposed as follows, for any fixed $\theta^*\in \K$, 
\begin{align*}
    \sum_{t=1}^T \big(\ell_t(\theta_t) - \ell_t(\theta^*)\big) &\le \sum_{t=1}^T \big(h_t(\theta_t) - h_t(\theta^*)\big) + \eps T \le \regret_T(\A) + \eps T.
\end{align*}
Taking $\theta^* \in \K$ to be the best decision in hindsight concludes the lemma proof.
\end{proof}
\begin{corollary}\label{cor:ogd_almost_convex}
Suppose $\{\ell_t\}_{t=1}^T$ are $\eps$-nearly convex and let $\A$ be OGD with stepsizes $\eta_t = \frac{2R}{G} \cdot t^{-1/2}$, then Algorithm \ref{alg:onco} has regret
\begin{align*}
    \sum_{t=1}^T \ell_t(\theta_t)  - \min_{\theta^* \in \K} \sum_{t=1}^T \ell_t(\theta^*) \le 3 R G \sqrt{T}  + \eps T,
\end{align*}
where $G$ is the gradient norm upper bound for all $\ell_t, t\in [T]$, and $R$ is the radius of $\K$.
\end{corollary} 
\section{Proofs for Section \ref{sec:online_twolayer}}\label{sec:app_warmup}
\paragraph{Neural Tangent Kernel.} The Neural Tangent Kernel (NTK) was first introduced in \citet{jacot2018neural}, who showed a connection between overparameterized neural networks and kernel methods. We characterize the net's expressivity by the capacity of learning functions in the RKHS of the NTK, which for our two-layer neural network has the following form:
\begin{definition}\label{def:shallow_ntk}
The NTK for the scalar two-layer neural network with activation $\sigma$ and intialization distribution $\theta \sim \N(0, I_p)$ is defined as $K_\sigma (x, y) = \E_{\theta \sim \N(0, I_p)}\langle x\sigma'(\theta^\top x), y\sigma'(\theta^\top y)\rangle$.
\end{definition}

Let $\mathcal{H}(K_\sigma)$ denote the RKHS of the NTK. Intuitively, $\mathcal{H}(K_\sigma)$ represents the space of functions that can be approximated by a neural network with kernel $K_\sigma$. To obtain non-asymptotic approximation guarantees, we focus on RKHS functions of bounded norm, specifically the RF-norm as defined below.

\begin{definition}[\citep{gao2019convergence}]\label{def:rf_space}
Consider functions of the form 
$$
h(x) = \int_{\reals^d}c(w)^\top x\sigma'(w^\top x)dw .
$$
Define the RF-norm of $h$ as $\|h\|_{RF} = \sup_w \frac{\|c(w)\|_2}{p_0(w)}$, where $p_0(w)$ is the probability density function of $\N(0, I_p)$. Let
\begin{align}\label{eq:rf_space}
\F_{RF}(D) = \{h(x) = \int_{\reals^d}c(w)^\top x\sigma'(w^\top x)dw \, : \, \|h\|_{RF} \le D\},
\end{align}
and extend to the multi-dimensional case, $
\F_{RF}^d(D) = \{h = (h_1, h_2, \ldots, h_d): h_i
\in \F_{RF}(D)\}
$.
\end{definition} 

By Lemma C.1 in \cite{gao2019convergence}, the class of Random Feature functions,  $\F_{RF}(\infty)$, is dense in $\H(K_\sigma)$ with respect to the $\|\cdot\|_{\infty, \mathbb{S}}$ norm, where $\|h\|_{\infty, \sphere} = \sup_{x\in \sphere_p} |h(x)|$. Since we are concerned with the approximation of the function value over the unit sphere, it is sufficient to consider $\F_{RF}^d(\infty)$, and further restrict to $\F_{RF}^d(D)$ for explicit nonasymptotic guarantees. The remaining of this section covers the proofs of the claims in Section \ref{sec:online_twolayer}. We remark that the scaling factor in \eqref{eq:shallow_nn} is optimally chosen to be $b = \sqrt{m}$ in the proof of Theorem \ref{thm:shallow_nn}.
\\

\begin{proof}[Proof of Theorem \ref{thm:shallow_nn}] Let $g\in \F_{RF}^d(D)$. By Lemma \ref{lem:shallow_approx}, with probability at least $1-\delta$ over the random initialization $\theta_1$, there exists $\theta^* \in B(R)$ such that for all $x\in \mathbb{S}_p$, 
\begin{align*}
 \ell_t(f(\theta^*;x)) &\le \ell_t(g(x)) + \frac{L\sqrt{d}bCD^2}{2m} + \frac{L\sqrt{d}CD}{\sqrt{m}}(2\sqrt{2p} + 2\sqrt{\log d/\delta})\\
 &\le \ell_t(g(x)) + \tilde{O}\left(\frac{L\sqrt{dp}CD^2}{\sqrt{m}}\right),
 \end{align*}
 using the optimal scaling factor choice $b = \sqrt{m}$.
 By the regret guarantee in Lemma \ref{lem:shallow_convergence}, Algorithm \ref{alg:ogd_nn} has regret 
\begin{align}
\sum_{t=1}^T \ell_t(\theta_t) &\le \min_{\theta \in B(R)} \sum_{t=1}^T \ell_t(\theta) + \frac{3 C L R \sqrt{mT}}{b} + \frac{2 C L R^2}{b} T\\
&=\min_{\theta \in B(R)} \sum_{t=1}^T \ell_t(\theta) + O(C L R \sqrt{T} + \frac{C L R^2}{\sqrt{m}} T).
\end{align}
Combining them and using $R = D \sqrt{d}$, we conclude
\begin{align*}
    \sum_{t=1}^T \ell_t(\theta_t) &\le \min_{\theta \in B(R)} \sum_{t=1}^T \ell_t(\theta) + O(C L R \sqrt{T} + \frac{C L R^2}{\sqrt{m}} T)\\
    &\le \sum_{t=1}^T \ell_t(\theta^*) + O(C L R \sqrt{T} + \frac{C L R^2}{\sqrt{m}} T)\\
    &\le \sum_{t=1}^T \ell_t(g(x_t)) + O(C L R \sqrt{T} + \frac{C L R^2}{\sqrt{m}} T) + \tilde{O}(\frac{L\sqrt{dp}CD^2T}{\sqrt{m}})
\end{align*}
The theorem follows by noticing that the inequality holds for any arbitrary $g\in \F_{RF}^d(D)$.
\end{proof}

\begin{proof}[Proof of Lemma \ref{lem:shallow_almost_convex}] We extend the original proof in \cite{gao2019convergence}. 
Let $\text{diag}(a_i)$ be a diagonal matrix with \\$(a_{1, i}, \ldots, a_{m/2, i}, -a_{1, i}, \ldots, -a_{m/2, i})$ on the diagonal. Note that the gradient of the network at the $i$-th coordinate is
\begin{align}\label{eq:net_gradient}
\nabla_{\theta[i]} f_i(\theta[i]; x) = \frac{1}{b}\text{diag}(a_i)\sigma'(\theta[i] x)x^\top.
\end{align}
We can show that the gradient is Lipschitz as follows, for all $x\in \sphere_p$,
\begin{align} \label{eq:lipschitz_grad}
    \|\nabla_{\theta[i]} f_i(\theta[i]; x) - \nabla_{\theta[i]} f_i(\theta'[i]; x)\|_F &\le \frac{1}{b}\|\text{diag}(a_i)\|_2\|\sigma'(\theta[i] x)- \sigma'(\theta'[i] x)\|_2\|x\|_2 \\
    &\le \frac{C}{b}\|\theta[i] - \theta'[i]\|_F. \tag{$|a_{r, i}| = 1, \|x\|_2 = 1$}
\end{align}
For each $\ell_t(f(\theta; x_t))$ according to the convexity property we have
\begin{align*}
    \ell_t(\theta') - \ell_t(\theta) &\ge \nabla_f \ell_t(\theta)^\top (f(\theta';x_t) - f(\theta;x_t))\\
    &= \sum_{i=1}^{d} \frac{\partial \ell_t(\theta)}{\partial f_i(\theta[i];x_t)} (f_i(\theta'[i];x_t) - f_i(\theta[i];x_t))
\end{align*}
For each $i \in [d]$, we use the fundamental theorem of calculus to rewrite function value difference as
\begin{align} \label{eq:almost_convex}
    f_i(\theta'[i];x_t) - f_i(\theta[i];x_t) &=  \langle \nabla_{\theta[i]} f_i(\theta[i];x_t), \theta'[i] - \theta[i] \rangle + \mathcal{R}(f_i, \theta[i], \theta'[i]) \\
    \mathcal{R}(f_i, \theta[i], \theta'[i]) &= \int_0^1 \langle \nabla_{\theta[i]} f_i(s\theta'[i] + (1-s)\theta[i];x_t) - \nabla_{\theta[i]} f_i(\theta[i];x_t), \theta'[i] - \theta[i] \rangle ds. \nonumber
\end{align}
Note that since the gradient of $f_i$ is Lipschitz given by \eqref{eq:lipschitz_grad}, the residual term is bounded in magnitude as follows,
\begin{align*}
    \lvert \mathcal{R}(f_i, \theta[i], \theta'[i]) \rvert \leq \int_0^1 \frac{C}{b} \|s(\theta'[i] - \theta[i]) \|_F \cdot \| \theta'[i] - \theta[i] \|_F ds = \frac{C}{2 b} \| \theta'[i] - \theta[i] \|_F^2.
\end{align*}
Hence we can show that the loss is nearly convex with respect to $\theta$,
\begin{align*}
    \ell_t(\theta') - \ell_t(\theta) &\ge \sum_{i=1}^{d} \frac{\partial \ell_t(\theta)}{\partial f_i(\theta[i];x_t)} (f_i(\theta'[i];x_t) - f_i(\theta[i];x_t))\\
    &= \sum_{i=1}^{d} \frac{\partial \ell_t(\theta)}{\partial f_i(\theta[i];x_t)} \left( \langle \nabla_{\theta[i]} f_i(\theta[i];x_t), \theta'[i] - \theta[i] \rangle + \mathcal{R}(f_i, \theta[i], \theta'[i]) \right)\\
    &\geq \sum_{i=1}^{d} \langle \frac{\partial \ell_t(\theta)}{\partial f_i(\theta[i];x_t)} \nabla_{\theta[i]} f_i(\theta[i];x_t), \theta'[i] - \theta[i] \rangle - \frac{C}{2 b} \sum_{i=1}^{d} \left\lvert \frac{\partial \ell_t(\theta)}{\partial f_i(\theta[i];x_t)} \right\rvert \cdot \|\theta'[i] - \theta[i]\|_F^2\\
    &\geq \langle \nabla_\theta \ell_t(\theta), \theta' - \theta\rangle - \frac{CL}{2 b} \|\theta' - \theta\|_F^2,
\end{align*}
where the last inequality uses the $L$-Lipschitz property of the loss $\ell_t(\cdot)$ with respect to $f$. Using a diameter bound for $\theta, \theta' \in B(R)$ we get that $\| \theta - \theta' \|_F \leq 2R$ which results in near convexity of $\ell_t(\cdot)$ with $\eps_{\text{nc}} = \frac{2 C L R^2}{b}$ with respect to $\theta$.
\end{proof}

\begin{proof}[Proof of Lemma \ref{lem:shallow_convergence}]
The theorem statement is shown by using Corollary \ref{cor:ogd_almost_convex} and showing that the loss functions $\ell_t : \reals^{d \times m \times p} \to \reals^{d}$ satisfy near-convexity with respect to $\theta$. First, the decision set in this case is $\K = B(R)$ so its radius is $R$. Lemma \ref{lem:shallow_almost_convex} shows that the loss functions $\ell_t(\theta)$ are $\eps_{\text{nc}}$-nearly convex with $\eps_{\text{nc}} = \frac{2 C L R^2}{b}$. Finally, we can show that the gradient norm is bounded as follows,
\begin{align*}
    \| \nabla_\theta \ell_t(\theta) \|_F^2 = \sum_{i=1}^{d} \| \nabla_{\theta[i]} \ell_t(\theta) \|_F^2 = \sum_{i=1}^{d} \left \lvert \frac{\partial \ell_t(\theta)}{\partial f_i(\theta[i]; x_t)} \right \rvert^2 \cdot \| \nabla_{\theta[i]} f_i(\theta[i];x_t) \|_F^2 \leq \frac{C^2 L^2m}{b^2},
\end{align*}
where we use the $L$-Lipschitz property of $\ell_t(f(\theta;x))$ and the fact that the $f_i$ gradient is bounded $\| \nabla_{\theta[i]} f_i(\theta[i];x_t) \|_F \leq \sqrt{m} C/b$ given \eqref{eq:net_gradient}. This means that $G = \frac{C L\sqrt{m}}{b}$ and we can use the Corollary \ref{cor:ogd_almost_convex} to conclude the final statement in \eqref{eq:ogd_nn_r}.
\end{proof}

\begin{lemma} \label{lem:approx}
For any $\delta, D > 0$, let $g:\reals^p \rightarrow \reals \in \F_{RF}(D)$ and let $R' = \frac{bD}{\sqrt{m}}$, then for a fixed $i\in[d]$, with probability at least $1-\delta$ over the random initialization $\theta_1$, there exists $\theta^* \in \reals^{m\times p}$ such that $\|\theta^* - \theta_1\|_F \le R'$, and for all $x\in \sphere_p$,
\begin{align*}
|f_i(\theta^*;x) - g(x)| \le \frac{bCD^2}{2m} + \frac{CD}{\sqrt{m/2}}(2\sqrt{p} + \sqrt{2\log 1/\delta}).
\end{align*}
\end{lemma}

\begin{proof}
Since the neural network architectures are the same for all $i\in [d]$, we fix an arbitrary $i$ and drop the index $i$ for $\theta[i]$ throughout the proof. By Proposition C.1 in \cite{gao2019convergence}, for any $\delta > 0$, with probability at least $1-\delta$ over the randomness of $\theta_1$, there exist $c_1, \cdots, c_{m/2} \in \reals^p$ with $\|c_r\|_2\le \frac{2\|g\|_{RF}}{m}\ \forall\ r\in[\frac{m}{2}]$, such that $g_1(x) = \sum_{r=1}^{m/2}c_r^\top x\sigma'((\theta_1[r])^\top x)$ satisfies
\begin{align*}
\forall\ x \in \mathbb{S}, \, |g_1(x) - g(x)|\le \frac{C\|g\|_{RF}}{\sqrt{m/2}}(2\sqrt{p} + \sqrt{2\log 1/\delta}),
\end{align*}
where $\theta_1[r]$ represents the $r$-th row of $\theta_1$. Now, we proceed to construct a $\theta^*$ such that $f_i(\theta^*; x)$ is close to $g_1(x)$. We note that by symmetric initialization $f_i(\theta_1;x) = 0$ for all $x \in \sphere_p$. Then, use the fundamental theorem of calculus similarly to \eqref{eq:almost_convex} to decompose $f_i$ as follows:
\begin{align*}
    f_i(\theta; x) &= f_i(\theta;x) - f_i(\theta_1; x) \\
    &= \frac{1}{b}\big(\sum_{r=1}^{m/2} a_{r}(\theta[r] - \theta_1[r])^\top x\sigma'((\theta_1[r])^\top x) - \sum_{r=1}^{m/2}a_{r}(\bar{\theta}[r] - \bar{\theta}_1[r])^\top x\sigma'((\bar{\theta}_1[r])^\top x)\big )\\
    &+ \frac{1}{b}\big(\sum_{r=1}^{m/2} a_{r}\int_0^1x^\top(\theta[r] - \theta_1[r]) (\sigma'((t\theta[r] + (1-t)\theta_1[r])^\top x) - \sigma'((\theta_1[r])^\top x))dt\\
    &-\sum_{r=1}^{m/2} a_{r}\int_0^1x^\top(\bar{\theta}[r] - \bar{\theta}_1[r]) (\sigma'((t\bar{\theta}[r] + (1-t)\bar{\theta}_1[r])^\top x) - \sigma'((\bar{\theta}_1[r])^\top x))dt\big ).
\end{align*}
Consider $\theta^* \in \reals^{m\times p}$ such that $\theta^*[r] = \theta_1[r] + \frac{b}{2}c_ra_{r},\ \bar{\theta}^*[r] = \bar{\theta}_1[r] - \frac{b}{2}c_ra_{r}$, where $\bar{\theta}^*[r]^\top$ represents the $\frac{m}{2}+r$-th row of $\theta^*$.  Then $$\|\theta^*[r] - \theta_1[r]\|_2,\ \|\bar{\theta}^*[r] - \bar{\theta}_1[r]\|_2 \le \frac{b\|g\|_{RF}}{m},\ \ \ \ \  \text{and the linear part of $f_i$ satisfies}$$
\begin{align*}
&\frac{1}{b}\big(\sum_{r=1}^{m/2} a_{r}(\theta^*[r] - \theta_1[r])^\top x\sigma'((\theta_1[r])^\top x) - \sum_{r=1}^{m/2}a_r(\bar{\theta}^*[r] - \bar{\theta}_1[r])^\top x\sigma'((\bar{\theta}_1[r])^\top x)\big )\\
&=\frac{1}{b}\big(\sum_{r=1}^{m/2} a_{r}^2\frac{b}{2}c_r^\top x\sigma'((\theta_1[r])^\top x) + \sum_{r=1}^{m/2}a_{r}^2\frac{b}{2}c_r^\top x\sigma'((\bar{\theta}_1[r])^\top x)\big )\\
&= \frac{1}{b}\big(\sum_{r=1}^{m/2} \frac{b}{2}c_r^\top x\sigma'((\theta_1[r])^\top x) + \sum_{r=1}^{m/2}\frac{b}{2}c_r^\top x\sigma'((\theta_1[r])^\top x)\big )\\
&=\sum_{r=1}^{m/2} c_r^\top x\sigma'((\theta_1[r])^\top x) = g_1(x).
\end{align*}
Now we bound the residual part of $f_i$, by using the triangle inequality, and the smoothness of $\sigma(\cdot)$, as follows
\begin{align*}
    |f_i(\theta^*; x) - g_1(x)| &= \frac{1}{b}\big \lvert \sum_{r=1}^{m/2} a_{r}\int_0^1x^\top(\theta^*[r] - \theta_1[r]) (\sigma'((t\theta^*[r] + (1-t)\theta_1[r])^\top x) - \sigma'((\theta_1[r])^\top x))dt\\
    &-\sum_{r=1}^{m/2} a_{r}\int_0^1x^\top(\bar{\theta}^*[r] - \bar{\theta}_1[r]) (\sigma'((t\bar{\theta}^*[r] + (1-t)\bar{\theta}_1[r])^\top x) - \sigma'((\bar{\theta}_1[r])^\top x))dt\big \rvert\\
    &\le  \frac{mC}{b}\frac{b^2}{4}\frac{4\|g\|_{RF}^2}{2m^2}= \frac{bC\|g\|_{RF}^2}{2m}.
\end{align*}
Using the triangle inequality, we can bound the approximation error as follows,
\begin{align*}
    |f_i(\theta^*;x) - g(x)| &\le |f_i(\theta^*; x) - g_1(x)| + |g_1(x) - g(x)|\\
    &\le \frac{bC\|g\|_{RF}^2}{2m} + \frac{C\|g\|_{RF}}{\sqrt{m/2}}(2\sqrt{p} + \sqrt{2\log 1/\delta}).
\end{align*}
Finally, observe that $\theta^*$ is close to $\theta_1$:
$$
\|\theta^* - \theta_1\|_F^2 \le \sum_{r=1}^m \|\theta^*[r] - \theta_1[r]\|_2^2 \le \frac{b^2\|g\|_{RF}^2}{m} \le \frac{b^2D^2}{m} = (R')^2.
$$
\end{proof}

\begin{proof}[Proof of Lemma \ref{lem:shallow_approx}] Let $g = (g_1, \ldots, g_d) \in \F_{RF}^d(D)$. By Lemma \ref{lem:approx}, if $R' = \frac{bD}{\sqrt{m}}$, with probability at least $1-\delta/d$, for each $i$ there exists $\theta^*[i]$ such that $\|\theta^*[i]- \theta_1[i]\|_F\le R'$, and 
$$|f_i(\theta^*[i]; x) - g_i(x)|\le \frac{bCD^2}{2m} + \frac{CD}{\sqrt{m/2}}(2\sqrt{p} + \sqrt{2\log d/\delta}).$$
Let $\theta^* = (\theta^*[1], \ldots, \theta^*[d])$. Taking a union bound, with probability at least $1-\delta$,
\begin{align*}
     \ell_t(f(\theta^*;x)) 
    &= \ell_t(f_1(\theta^*[1];x),\ldots, f_{d}(\theta^*[d];x))\\
    &\le \ell_t(g_1(x), \ldots, g_{d}(x)) + L\sqrt{\sum_{i=1}^{d}\big(f_i(\theta^*[i];x) - g_i(x)\big)^2}\\
    &\le \ell_t(g(x)) + \frac{Lb\sqrt{d}CD^2}{2m} + \frac{L\sqrt{d}CD}{\sqrt{m/2}}(2\sqrt{d} + \sqrt{2\log d/\delta}).
\end{align*}
Finally, observe that $\|\theta^* - \theta_1\|_F \le \sqrt{d}R' = R$.
\end{proof}
\section{Proofs for Section \ref{sec:dnn}}

\begin{proof}[Proof of Theorem \ref{thm:main_nn}]
To prove this theorem, we will use both Lemmas \ref{lem:deep_convergence} and \ref{lem:nn_interpol}. First, let us verify that the conditions of Lemma \ref{lem:deep_convergence}, i.e. conditions of Lemma \ref{lem:deep_nc}, are satisfied given the choice of $m, R$ in the theorem statement. Indeed, under our choice of $m$, as long as $ m \ge \frac{c_1 k^6\log^8 mH^{12}}{\gamma^2}$ for some sufficiently large $c_1 > 0$, we have $\frac{k^3\log m}{\gamma \sqrt{m}} \le \frac{1}{\sqrt{c_1} H^6\log^3 m}$. 
Suppose for some constant $c_2$, taking
$$
R = \frac{c_2 k^3\log m}{\gamma\sqrt{m}}
$$
satisfies the condition required for Lemma \ref{lem:nn_interpol}. Then we can set $c_1$ to be large enough such that $\frac{c_2}{\sqrt{c_1}} \le c'$ for $c'$ specified in Lemma \ref{lem:deep_nc}, and choosing
$$
m \ge \frac{c_1 p^{3/2}(k^{24}H^{12}\log^8 m + d)^{3/2}}{\gamma^8} \ge \Omega \left( \frac{k^{24} H^{12} \log^5 m}{\gamma^8} \right)
$$
gives us an $R$ that satisfies the Lemma \ref{lem:deep_nc}'s condition.

For the condition on $m$, we simply have
\begin{align*}
mR^{2/3}H = \frac{c_2^{2/3}k^2m^{2/3}H\log^{2/3} m}{\gamma^{2/3}} &\ge \frac{(c_1c_2)^{2/3}p(k^{24}H^{12}\log^8 m + d)k^2H\log^{2/3} m}{\gamma^6}\\
&\ge (c_1c_2)^{2/3}p(k^{24}H^{12}\log^8 m + d)\\
&\ge \Omega(p\log O(1/R) + \log d).
\end{align*}
Observe that under these choices of $m, R$ the conditions from Lemma \ref{lem:nn_interpol} are trivially satisfied. Hence, we plug in the value of $R$ into the regret bound \eqref{eq:deep_regret_param} in Lemma \ref{lem:deep_convergence} and use Lemma \ref{lem:nn_interpol} to conclude the final regret bound in Theorem \ref{thm:main_nn}.
Finally, note that $mR^{2/3}H = \Omega(\log^2 m)$, and by taking a union bound over the events of Lemma \ref{lem:nn_interpol} and Lemma \ref{lem:deep_convergence}, the failure probability for the regret bound is
\begin{align*}
    d \cdot e^{-\Omega(\log^2 m)} + O(H) \cdot e^{-\Omega(m R^{2/3} H)} = O(H + d) \cdot e^{-\Omega(\log^2 m)} ~.
\end{align*}
This concludes the theorem, verifying that the failure probability is low, since $m \gg \max(d, H)$.
\end{proof}

\begin{proof}[Proof of Lemma \ref{lem:deep_nc}]
Our proof extends Lemma A.6 in \cite{gao2019convergence} to our setting, where the loss is defined over a vector whose coordinates are outputs of different deep neural networks. A $\delta$-net over $\mathbb{S}_p$ is defined as a collection of points $\{x_r\}\in \mathbb{S}_p$ such that for all $x\in \mathbb{S}_p$, there exists an $x_j$ in the $\delta$-net such that $\|x_j - x\|_2\le \delta$. Consider a $\delta$-net of the unit sphere consisting of $\{x_r\}_{r=1}^N$, and standard results show that such a $\delta$-net exists with $N = (O(1/\delta))^p$. Let $i\in [d]$ and $r\in [N]$. By Lemma A.5 in \cite{gao2019convergence}, if $m \ge \max\{d, \Omega(H\log H)\}$, $R + \delta \le \frac{c}{H^6\log^3 m}$ for some sufficiently small constant $c$, then with probability at least $1-O(H)e^{-\Omega(m(R+\delta)^{2/3}H)}$ over the random initialization, for any $\theta'[i], \theta[i] \in B(R)$ and any $x'\in \mathbb{S}_p$ with $\|x' - x_r\|_2 \le \delta$, 
$$
\|\nabla_{\theta^h[i]} f_i(\theta'[i]; x') - \nabla_{\theta^h[i]} f_i(\theta[i]; x')\|_F = O((R+\delta)^{1/3}H^2\sqrt{m\log m}),
$$
$$
\|\nabla_{\theta^h[i]} f_i(\theta'[i]; x')\|_F = O(\sqrt{mH}),
$$
where $\theta^h[i]$ denotes the parameter for layer $h$ in the network for the $i$-th coordinate of the output. Summing over the layers, we have
$$
\|\nabla_{\theta[i]} f_i(\theta'[i]; x') - \nabla_{\theta[i]} f_i(\theta[i]; x')\|_F = O((R+\delta)^{1/3}H^{5/2}\sqrt{m\log m}),
$$
$$
\|\nabla_{\theta[i]} f_i(\theta'[i]; x')\|_F = O(H\sqrt{m}).
$$
Similar to \eqref{eq:almost_convex}, we can write the difference of $f_i$ evaluated on $\theta'[i]$ and $\theta[i]$ as a sum of a linear term and a residual term $\mathcal{R}(f_i, \theta[i], \theta'[i], x')$ using the Fundamental Theorem of Calculus,
\begin{align}\label{eq:deep_decomp}
f_i(\theta'[i]; x') - f_i(\theta[i];x') &=  \langle \nabla_{\theta[i]} f_i(\theta[i];x'), \theta'[i] - \theta[i] \rangle + \mathcal{R}(f_i, \theta[i], \theta'[i], x') \\
\mathcal{R}(f_i, \theta[i], \theta'[i], x') &= \int_0^1 \big \langle \nabla_{\theta[i]} f_i(s\theta'[i] + (1-s)\theta[i]; x') - \nabla_{\theta[i]} f_i(\theta[i]; x'), \theta'[i] - \theta[i] \big \rangle ds
\end{align}
Since we can bound the change of the gradient, we can bound the residual term as follows
\begin{align*}
    \lvert \mathcal{R}(f_i, \theta[i], \theta'[i], x')\rvert &\le \int_0^1 \big \| \nabla_{\theta[i]} f_i(s\theta'[i] + (1-s)\theta[i]; x') - \nabla_{\theta[i]} f_i(\theta[i]; x')\|_F \|\theta'[i] - \theta[i] \|_F ds\\
&\le O\big((R+\delta)^{1/3}H^{5/2}\sqrt{m\log m}\big)\|\theta'[i] - \theta[i]\|_F.
\end{align*}
Taking a union bound over the $i$'s,  with probability at least $1-O(H)de^{-\Omega(m(R+\delta)^{2/3}H)}$, for all $x'$ such that $\|x' - x_r\|_2\le\delta$,
\begin{align*}
    \ell_t(f(\theta'; x')) - \ell_t(f(\theta;x')) &\ge \sum_{i=1}^{d} \frac{\partial \ell_t(f(\theta; x'))}{\partial f_i(\theta[i]; x')} (f_i(\theta'[i]; x') - f_i(\theta[i]; x'))\\
    &= \sum_{i=1}^{d} \frac{\partial \ell_t(f(\theta; x'))}{\partial f_i(\theta[i]; x')} \left( \langle \nabla_{\theta[i]} f_i(\theta[i]; x'), \theta'[i] - \theta[i] \rangle + \mathcal{R}(f_i, \theta[i], \theta'[i], x') \right)\\
    &\geq \sum_{i=1}^{d} \langle \frac{\partial \ell_t(f(\theta; x'))}{\partial f_i(\theta[i]; x')} \nabla_{\theta[i]} f_i(\theta[i]; x'), \theta'[i] - \theta[i] \rangle \\
    &- O\big((R+\delta)^{1/3}H^{5/2}\sqrt{m\log m}\big) \sum_{i=1}^{d} \left \lvert \frac{\partial \ell_t(f(\theta; x'))}{\partial f_i(\theta[i]; x')}\right\rvert \cdot \|\theta'[i] - \theta[i]\|_F\\
    &\geq \langle \nabla_\theta \ell_t(f(\theta; x')), \theta' - \theta\rangle - O\big((R+\delta)^{1/3}H^{5/2}\sqrt{m\log m}\big)L\sqrt{d} R.
\end{align*}
We take $\delta = R$, and by our choice of $R$, the condition $R + \delta \le \frac{c}{H^6\log^3 m}$ is satisfied. Taking a union over bound all points in the $\delta$-net, the above inequality holds for all $x\in \mathbb{S}_p$ with probability at least 
\begin{align*}
    1-dO(H)O(1/R)^pe^{-\Omega(mR^{2/3}H)}  &= 1-O(H)e^{-\Omega(mR^{2/3}H) + p\log(O(1/R)) + \log d}\\
    &=1-O(H)e^{-\Omega(mR^{2/3}H)},
\end{align*}
where the last inequality is due to our choice of $m$. This applies to the gradient bound too, i.e. 
\begin{align}\label{eq:deep_grad_bound}
\|\nabla_{\theta[i]} f_i(\theta[i]; x)\|_F = O(H\sqrt{m}), \ \forall i \in [d],
\end{align}
holds for any $\theta \in B(R)$ and any $x \in \sphere_p$ with the same failure probability.
\end{proof}

\begin{proof}[Proof of Lemma \ref{lem:deep_convergence}]
Given that the identical conditions of Lemma \ref{lem:deep_nc} hold, then with probability at least $1-O(H)e^{-\Omega(mR^{2/3}H)}$ over the randomness of $\theta_1$, $\ell_t$ is $\eps_{\text{nc}}$-nearly convex with  $\eps_{\text{nc}} = O(R^{4/3}H^{5/2}\sqrt{m\log m}L\sqrt{d})$, and $\|\nabla_{\theta[i]}f_i(\theta[i];x)\|_F \le O(H\sqrt{m})$ according to \eqref{eq:deep_grad_bound} for all $i\in [d], x\in \mathbb{S}_p, \theta\in B(R)$. Since the decision set is $B(R)$, its radius in Frobenius norm is at most $R\sqrt{d}$. We can bound the gradient norm as follows, for all $x\in \mathbb{S}_p$,
\begin{align*}
   \|\nabla_\theta \ell_t(f(\theta;x))\|_F^2 &= \sum_{i=1}^d  \|\nabla_{\theta[i]} \ell_t(f_i(\theta[i];x))\|_F^2 \\
   &= \sum_{i=1}^d \left|\frac{\partial \ell_t(f(\theta;x))}{\partial f_i(\theta[i];x)}\right|^2 \cdot\|\nabla_{\theta[i]}f_i(\theta[i];x)\|_F^2\\
   &\le L^2 \max_i \|\nabla_{\theta[i]}f_i(\theta[i];x)\|_F^2\le O(L^2H^2m).
\end{align*}
By Corollary \ref{cor:ogd_almost_convex}, the regret is bounded by 
$$
3R\sqrt{d}G\sqrt{T} + \eps T \le O(RLH\sqrt{dmT}) + O(R^{4/3}H^{5/2}TL\sqrt{dm\log m}) ~.
$$
which concludes the proof.
\end{proof}



\subsection{Auxiliary Lemmas}

\begin{lemma} \label{lem:bounded_output_supp}
For $m \ge \Omega(\frac{p\log(1/R)+\log (d/\delta)}{R^{2/3}H})$, and $R = O(\frac{1}{H^6\log^3 m})$, with probability at least $1-\delta$ over the randomness of initialization, for all $x\in \mathbb{S}_p$ and all $i\in[d]$, $|f_i(\theta_1[i];x)|\le O\left( \sqrt{\log \frac{d}{\delta}} + \sqrt{p\log \frac{1}{R}}\right)$.
\end{lemma}
\begin{proof}
As in the proof of Lemma \ref{lem:deep_nc}, we consider an $\eps$-net consisting of $O(1/\eps)^p$ points over the unit sphere in dimension $p$, and fix $x_r$ in the $\eps$-net. Let $i\in[d]$, and define $B_i(R) = \{\theta[i]:\|\theta[i] - \theta_1[i]\|_2\le R\}$. Let $f_i^h(\theta[i]; x)$ denote output at the $h$-th layer of the network after activation, with weights $\theta[i]$ and input $x$. 

By Lemma A.4 in \citet{gao2019convergence}, if $R = O(1)$, with probability $1-O(H)e^{-\Omega(m/H)}$ over random initialization, for any $x'\in \mathbb{S}_p$ such that $\|x_r - x'\|_2\le \eps$, and any $\theta[i]\in B_i(R)$, in particular $\theta_1[i]$, there exists $\tilde{\theta}[i]\in B_i(R+O(\eps))$ such that 
$$
f_i^H(\tilde{\theta}[i];x_r) = f_i^H(\theta_1[i];x').
$$
We first decompose the output of the neural net as follows,
\begin{align*}
    |f_i(\theta_1[i];x')| = |a^\top f_i^H(\theta_1[i];x')| &= |a^\top f_i^H(\tilde{\theta}[i];x_r)|\\
    &\le |a^\top( f_i^H(\tilde{\theta}[i];x_r) - f_i^H(\theta_1[i];x_r))| +|a^\top f_i^H(\theta_1[i];x_r))|.
\end{align*}

Note that since $a\sim \N(0, I_m)$, for any fixed vector $v$, we have $a^\top v \sim \N(0, \|v\|_2^2)$. By Hoeffding's inequality, for all $c\ge 0$
$$
\Pr[|a^\top v| \ge c\|v\|] \le 2e^{-\frac{c^2}{2}}.  
$$


Now we bound the first term.
According to Lemma 8.2 in \citet{allenzhu2019convergence}, for $R + O(\eps)\le \frac{c'}{H^6\log^3 m}$ for some sufficiently small $c'$, with probability at least $1-e^{-\Omega(m(R+O(\eps))^{2/3} H)}$, $\|f_i^H(\tilde{\theta}[i]; x_r) - f_i^H(\theta_1[i]; x_r)\|_2 \le c_1\cdot (R+O(\eps))H^{5/2}\sqrt{\log m}$ for some constant $c_1$.
Under this event, with probability at least $1-\delta'$, 
\begin{align*}
    |a^\top( f_i^H(\tilde{\theta}[i];x_r) - f_i^H(\theta_1[i];x_r))| &\le \sqrt{2\ln\left(\frac{2}{\delta'}\right)}c_1 (R+O(\eps))H^{5/2}\sqrt{\log m}\\
    &= O \left(\sqrt{\ln \frac{1}{\delta'}}(R+O(\eps))H^{5/2}\sqrt{\log m}\right).
\end{align*}

For the second term, by Lemma A.2 in \citet{gao2019convergence}, with probability at least $1-O(H)e^{-\Omega(m/H)}$ over the randomness of $\theta_1[i]$, $\|f_i^H
(\theta_1[i];x_r)\|_2 \le c_2$ for some constant $c_2$. Under this event, with probability at least $1-\delta'$,
\begin{align*}
    |a^\top f_i^H(\theta_1[i];x_r))| \le O\left(\sqrt{\ln \frac{1}{\delta'}}\right).
\end{align*}
We take $\eps = R$, and $R = O(\frac{1}{H^6\log^3 m})$, then the conditions on $R$ and $\eps$ are satisfied. 

We set $\delta' = \frac{\delta O(R)^p}{d}$, with our choice of $m$ and $R$, $O(H)e^{-\Omega(m/H)} = e^{-\Omega(m/H)}$, and $e^{-\Omega(mR^{2/3}H)} \le \delta'$.
Taking a union bound on the mentioned events, with probability at least $1-\delta'$, 
$$
|f_i(\theta_1[i];x')| \le O\left(\sqrt{\ln \frac{1}{\delta'}} RH^{5/2}\sqrt{\log m}\right) + O\left(\sqrt{\ln \frac{1}{\delta'}}\right) = O\left( \sqrt{\ln \frac{d}{\delta}} + \sqrt{p\ln\frac{1}{R}}\right).
$$
Now take a union bound over the $\eps$-net and over the $d$ coordinates, we conclude that for all $x\in \mathbb{S}_p$, for all $i\in [d]$
$
|f_i(\theta_1[i];x)| \le O\left( \sqrt{\ln \frac{d}{\delta}} + \sqrt{p\ln\frac{1}{R}}\right)
$
with probability at least 
\begin{align*}
    1 - dO(1/R)^p\delta' &= 1 - \delta.
\end{align*}
\end{proof}

\begin{lemma}\label{lem:bounded_output}
For $m \ge \Omega(\frac{p^{3/2}(k^{24}H^{12}\log^8 m + d)^{3/2}}{\gamma^8})$, and $R = O\left(\frac{k^3 \log m}{\gamma \sqrt{m}}\right)$, with probability at least $1-O(H+d)e^{-\Omega(\log^2 m)}$ over the randomness of initialization, for all $x\in \mathbb{S}_{p}$ and all $\theta \in B(R)$, for all $i\in [d]$, $|f_i(\theta[i];x)| \le O\left( \frac{k^3 (H+\sqrt{p}) \log m}{\gamma} \right)$. 
\end{lemma}

\begin{proof}
Observe that for each $i \in [d]$ and any $x \in \sphere_p$, the inequality 
$$|f_i(\theta[i] ; x)| \leq |f_i(\theta_1[i] ; x)| + |f_i(\theta[i] ; x) - f_i(\theta_1[i];x)|$$
holds. The choice of $m, R$ satisfies the conditions in Lemma \ref{lem:bounded_output_supp}, take $\delta = d e^{- \Omega(\log^2 m)}$ and note that $m R^{2/3} H  = \Omega(\log^2 m)$. We can use the decomposition in $\eqref{eq:deep_decomp}$ to bound the difference between the neural network output at $\theta[i]$ and that at $\theta_1[i]$.
\begin{align*}
    f_i(\theta[i];x) - f_i(\theta_1[i];x) &= \int_0^1 \big \langle \nabla_{\theta[i]} f_i(s\theta[i] + (1-s)\theta_1[i]; x), \theta[i] - \theta_1[i] \big \rangle ds
\end{align*}
By Lemma \ref{lem:deep_nc}, with our choice of $m$ and $R$, with probability at least $1-O(H)e^{-\Omega(\log^2 m)}$, 
$$
\|\nabla_{\theta[i]} f_i(s\theta[i]+(1-s)\theta_1[i]; x)\|_F = O(H\sqrt{m}),\ \forall\ s\in[0, 1].
$$
Therefore the integral can be bounded as 
\begin{align*}
    &\left |\int_0^1 \big \langle \nabla_{\theta[i]} f_i(s\theta[i] + (1-s)\theta_1[i]; x), \theta[i] - \theta_1[i] \big \rangle ds\right | \\
    &\le \int_0^1 \| \nabla_{\theta[i]} f_i(s\theta[i] + (1-s)\theta_1[i]; x)\|_F \|\theta[i] - \theta_1[i]\|_Fds\\
    &\le O(RH\sqrt{m}).
\end{align*}
Combining with Lemma \ref{lem:bounded_output_supp}, we conclude that 
\begin{align*}
    |f_i(\theta[i];x)| \le |f_i(\theta[i];x) - f_i(\theta_1[i];x)| + | f_i(\theta_1[i];x)| \le O\left(\frac{k^3 (H+\sqrt{p}) \log m}{\gamma}\right), \quad \forall i \in [d]
\end{align*}
with probability at least $1-O(H+d)e^{-\Omega(\log^2 m)}$.
\end{proof}

\section{Proofs for Section \ref{sec:control}}\label{sec:app_control}
\begin{proof}[Proof of Lemma \ref{lem:control_openloop}]
This lemma is shown by reducing it to the interpolation dimension lemma for deep neural networks, Lemma \ref{lem:nn_interpol}. The class of policies $\Pi_{\text{dnn}}(f; \Theta)$ is at the same time a hypothesis class of functions of type $\R^{K \cdot d_x + 1} \to \R^{d_u}$, i.e. $p = K \cdot d_x + 1$, $d = d_u$. Observe that the domain is still the unit sphere $\X = \sphere_{K \cdot d_x + 1}$ given the normalization of inputs ${\bar{z}_k}+{k=1}^K$. Furthermore, the inputs are separated in $\ell_2$ norm by $\gamma > 0$ for $\gamma = \frac{1}{2 K W + H}$:
$$
\forall k \in [K], \|z_k\|_2^2 \le K \cdot W^2 + K^2 \leq K^2 (W^2 + 1) \leq 4 K^2 W^2,
$$
assuming $W = \max(1, W)$ since $\max_{k \in [K]}\|w_k\|_2 \le W$ according to Assumption \ref{assumption:control_bound}. This means that
$$
\forall j, l \in [K], \ \|\bar{z}_j - \bar{z}_l\|_2^2 \ge \left( \frac{k}{\|z_k\|_2} - \frac{l}{\|z_l\|_2} \right) \ge \frac{1}{4 K^2 W^2},
$$
so taking $\gamma = \frac{1}{2KW+H}$ satisfies separability and condition in Lemma \ref{lem:nn_interpol}. Finally, the conditions on $m, R$ coincide with those in Lemma \ref{lem:nn_interpol} for $\gamma = \frac{1}{2KW+H}$ and interpolation dimension $K$. Hence, according to Lemma \ref{lem:nn_interpol}, the function class $\Pi_{\text{dnn}}(f; \Theta)$, with probability $1 - d_u \cdot e^{-\Omega(\log^2 m)}$, has interpolation dimension $\I_{\gamma}(\Pi_{\text{dnn}}(f; \Theta)) \ge K$ which directly implies that it can output any open-loop control sequence $u^*_{1:K}$ of length $K$ up to arbitrary precision, as stated in this lemma.
\end{proof}

\begin{proof}[Proof of Theorem \ref{thm:control}] The proof is very similar to that of Theorem \ref{thm:main_nn}. The theorem conditions are at least as strong as those in Lemma \ref{lem:control_nc}, hence we can use Lemma \ref{lem:control_nc} to claim that $\mL_t(\theta)$ is $\eps_{nc}$-nearly convex with $\eps_{nc} = O(L_c R^{4/3} H^{9/2} K^6 W d_u \sqrt{d_x m} \log^{3/2} m)$, and $\|\nabla_{\theta[i]}f_i(\theta[i];\bar{z}_k^t)\|_F \le O(H\sqrt{m})$ for all $i\in [d], \bar{z}_k^t\in \mathbb{S}_{K\cdot d_x+1}, \theta\in B(R; \theta_1)$. We first bound the gradient norm of $\mL_t(\theta)$:
\begin{align*}
   \|\nabla_\theta \mL_t(\bar{f}(\theta))\|_F^2 &= \| \sum_{k=1}^K \sum_{i=1}^{d_u}\frac{\partial \mL(\theta)}{\partial f_i(\theta[i];\bar{z}_k^t)} \nabla_{\theta[i]} f_i(\theta[i];\bar{z}_k^t)\|_F^2 \\
   &\le \sum_{k=1}^K \sum_{i=1}^{d_u} \left|\frac{\partial \mL_t(\theta)}{\partial f_i(\theta[i];\bar{z}_k^t)}\right|^2 \cdot\|\nabla_{\theta[i]}f_i(\theta[i];\bar{z}_k^t)\|_F^2\\
   &\le O(KL_c'^2) \max_{i, k} \|\nabla_{\theta[i]}f_i(\theta[i];\bar{z}_k^t)\|_F^2\\
   &\le O(K^{11} L_c^2 H^6 W^2 d_u d_x m \log^2 m),
\end{align*}
where the second to last inequality is due to Lemma \ref{lem:control_lip} and the last inequality holds because $L_c' = O(K^5 L_c H^2 W \sqrt{d_x d_u} \log m)$. We can proceed to bound the regret as follows
\begin{align*}
    3R\sqrt{d_u}G\sqrt{T} + \eps_{nc}T &\le O(R L_c K^{11/2} H^3 W d_u \sqrt{d_x m} \log m \sqrt{T}) + \\
     &+ O(R^{4/3} L_c K^6 H^{9/2} W d_u \sqrt{d_x m} \log^{3/2} m T)\\
    &= \tilde{O}(K^{19/2} L_c H^4 W^2 d_u \sqrt{d_x} \cdot \sqrt{T}) + \tilde{O}(\frac{K^{34/3} L_c H^{35/6} W^{7/3} d_u \sqrt{d_x}}{m^{1/6}} \cdot T) \\
    &= \tilde{O}(K^{10} L_c H^4 W^2 d_u d_x^{1/2} \cdot \sqrt{T}) + \tilde{O}\left( \frac{K^{12} L_c H^6 W^3 d_u d_x^{1/2}}{m^{1/6}} \cdot T \right).
\end{align*}
\end{proof}
\paragraph{Dynamics rollout.} Before proving the lemmas necessary for the theorem proof, we rewrite the state $x_k^{\theta}$ by rolling out the dynamics from $i = k$ to $i=1$ as follows
\begin{align*}
    x_k^{\theta} = x_k^{\text{nat}} + \sum_{i=1}^{k-1} M_i^k f(\theta; \bar{z}_i), \, x_k^{\text{nat}} = \prod_{j=k-1}^1 A_j x_1 + \sum_{i=1}^{k-1} \prod_{j=k-2}^{i} A_j w_i, \, M_i^k = \prod_{j=k-1}^{i+1} A_j \cdot B_i,
\end{align*}
 and for simplicity $\|x_1\|_2 \le W$. 
 
 \paragraph{Sequential stabilizability.} Furthermore, note that Assumption \ref{assumption:control_stable} can be relaxed to assuming there exists a sequence of linear operators $F_{1:K}$ such that for $C_1 \ge 1$ and $\rho_1 \in (0, 1)$
 \begin{align*}
     \forall k \in [K], n \in [1, k), \quad \left\| \prod_{i=k}^{k-n+1} (A_i + B_iF_i) \right\|_{\text{op}} \leq C_1 \cdot \rho_1^n ~.
 \end{align*}
 This condition is called {\it sequential stabilizability} and it reduces to the stable case by taking the actions $u'_k = F_k x_k + u_k$, yielding the stable dynamics of $(A_k+B_kF_k, B_k)_{1:K}$.
\begin{lemma}\label{lem:control_convex}
The function $\mL(\bar{f}(\theta))$ is convex in $\bar{f}(\theta)$.
\end{lemma}
\begin{proof}
The function $\mL(\bar{f}(\theta))$ is a sum of $K$ functions. For an arbitrary $k \in [K]$, note that $x_k^{\theta}$ is a affine function of $\bar{f}(\theta)$ w.r.t. the components $f(\theta, \bar{z}_i), i = 1, \dots, K$. The other argument is $f(\theta; \bar{z}_k)$ which is also an affine function of $\bar{f}(\theta)$. Hence, both arguments in $c_k(\cdot, \cdot)$, which is jointly convex in its arguments, are affine in $\bar{f}(\theta)$, which means that $c_k(x_k^{\theta}, f(\theta; \bar{z}_k))$ is convex in $\bar{f}(\theta)$. Since $\mL(\bar{f}(\theta))$ is defined as the sum over $c_k(x_k^{\theta}, f(\theta; \bar{z}_k))$, it is also convex in the argument $\bar{f}(\theta)$.
\end{proof}

\begin{lemma}
Under the identical conditions of Lemma \ref{lem:bounded_output}, the states and actions over an episode are bounded, $\max_k \| u_k^{\theta} \|_2 \le D_u$ and $\max_k \| x_k^{\theta} \|_2 \le D_x$ for $D_u = O(K^5 H^2 W \sqrt{d_u d_x} \log m)$, $D_x = \frac{C_1}{1-\rho_1} \cdot (W + D_u C_2)$.
\end{lemma}
\begin{proof}
First, note that $u_k^{\theta} = f(\theta; \bar{z}_k)$ and $\bar{z}_k \in \sphere_{K \cdot d_x+1}$. Given the output magnitude bound for the network in Lemma \ref{lem:bounded_output}, i.e. $\|u_k^{\theta}[i]\| \leq O(K^3 (H+\sqrt{Kd_x+1}) (2KW+H)\log m)$ we have $\| u_k^{\theta} \|_2 \leq O(\sqrt{d_u} K^3 (H+\sqrt{Kd_x+1})(2KW+H)\log m) = O(K^5 H^2 W \sqrt{d_u d_x} \log m) = D_u$. By definition of $x_k^{\text{nat}}$, we have that
\begin{align*}
    \| x_k^{\text{nat}} \|_2 \le W \cdot \frac{C_1}{1-\rho_1}
\end{align*}
Plugging this bound in the expression for $x_k^{\theta}$, we get
\begin{align*}
    \| x_k^{\theta} \|_2 \le W \cdot \frac{C_1}{1-\rho_1} + D_u \cdot \sum_{i=1}^{k-1} C_2 \cdot C_1 \cdot \rho_1^{k-i-1} \leq \frac{C_1}{1-\rho_1} \cdot (W + D_u C_2). 
\end{align*}
\end{proof}

\begin{corollary}\label{cor:lipschitz_control}
The cost function $c_k$ is $L'_c$-Lipschitz with $L'_c = L_c \cdot \max \{1, D_x + D_u \}$.
\end{corollary}

\begin{lemma}\label{lem:control_lip}
The function $\mL(\bar{f}(\theta))$ is $L$-Lipschitz w.r.t. each $f(\theta; \bar{z}_k)$ for $k \in [K]$ with $L = L'_c \cdot \frac{C_2 \cdot C_1}{1-\rho_1}$, i.e. $L = O(K^5 L_c H^2 W \sqrt{d_x d_u} \log m)$ under the identical conditions of Lemma \ref{lem:bounded_output}.
\end{lemma}
\begin{proof}
We use Corollary \ref{cor:lipschitz_control} with $L'_c$ to conclude this lemma statement. For any arbitrary $k \in [K]$, denote $f_k = f(\theta; \bar{z}_k)$ and note that in the expression of $\mL(\bar{f}(\theta))$ we have
\begin{align*}
    &\forall i < k, \quad \| \nabla_{f_k} c_i(x_k^{\theta}, u_k^{\theta}) \|_2 = 0, \\
    &\text{for } i = k, \quad \| \nabla_{f_k} c_i(x_k^{\theta}, u_k^{\theta}) \|_2 = \| \nabla_u c_i(x_k^{\theta}, u_k^{\theta}) \|_2 \leq L'_c, \\
    &\forall i > k, \quad \| \nabla_{f_k} c_i(x_k^{\theta}, u_k^{\theta}) \|_2 = \| (M_k^i)^{\top} \nabla_x c_i(x_k^{\theta}) \|_2 \le \| M_k^i \|_{\text{op}} \cdot L'_c
\end{align*}
Therefore, we conclude that
\begin{align*}
    \| \nabla_{f_k} \mL \|_2 \le \sum_{i=1}^K \| \nabla_{f_k} c_i \|_2 \le L'_c \cdot \sum_{i\geq k} \| M_k^i \|_{\text{op}} \leq L'_c \cdot \frac{C_2 \cdot C_1}{1-\rho_1} ~.
\end{align*}
\end{proof}



\begin{lemma}\label{lem:control_nc}
For $m \ge \Omega((K^{25} H^{12} d_x d_u \log^{8} m)^{3/2} (2KW+H)^8)$, and $R = O\left(\frac{K^3 (2KW+H) \log m}{\sqrt{m}}\right)$, with probability at least $1-O(H+d_u) e^{-\Omega(\log^2 m)}$ over the randomness of initialization $\theta_1$,
the loss $\mL(\theta) = \mL(\bar{f}(\theta))$, for any $\theta \in B(R; \theta_1)$ and any $\bar{z} \in \sphere_{Kd_x+1}$, is $\eps_{\text{nc}}$-nearly convex with $\eps_{\text{nc}} = O(L_c R^{4/3} H^{9/2} K^6 W d_u \sqrt{d_x m} \log^{3/2} m)$.

\end{lemma}
\begin{proof}
Since $\mL$ is convex in $\bar{f}$ by Lemma \ref{lem:control_convex}, we have that
\begin{align*}
    \mL(\bar{f}(\theta')) - \mL(\bar{f}(\theta)) &\ge \nabla_{\bar{f}} \mL(\bar{f}(\theta)) ^\top (\bar{f}(\theta') - \bar{f}(\theta)) \\
    &= \sum_{k=1}^{K}\sum_{j=1}^{d_u} \frac{\partial \mL}{\partial f_j(\theta; \bar{z}_k)} (f_j(\theta'; \bar{z}_k) - f_j(\theta; \bar{z}_k))
\end{align*}
Using the linearization trick as in \eqref{eq:deep_decomp}, we can write 
\begin{align*}
     \mL(\bar{f}(\theta')) - \mL(\bar{f}(\theta)) &\ge
    \sum_{k=1}^{K}\sum_{j=1}^{d_u} \frac{\partial \mL}{\partial f_j(\theta[j]; \bar{z}_k)} (\langle \nabla_{\theta[j]}f_j(\theta[j]; \bar{z}_k), \theta'[j] - \theta[j]\rangle + \mathcal{R}(f_j, \theta[j], \theta'[j], \bar{z}_k)).
\end{align*}
Pulling out the first term in the sum, we have
\begin{align*}
    &\sum_{k=1}^{K}\sum_{j=1}^{d_u} \frac{\partial \mL}{\partial f_j(\theta[j]; \bar{z}_k)} \langle \nabla_{\theta[j]}f_j(\theta[j]; \bar{z}_k), \theta'[j] - \theta[j]\rangle\\
    &=\sum_{j=1}^{d_u}\langle \sum_{k=1}^{K} \frac{\partial \mL}{\partial f_j(\theta[j]; \bar{z}_k)}  \nabla_{\theta[j]}f_j(\theta[j]; \bar{z}_k), \theta'[j] - \theta[j]\rangle \\
    &= \sum_{i=1}^{d_u}\langle \nabla_{\theta[j]} \mL(\theta), \theta'[j] - \theta[j]\rangle = \langle \nabla_\theta \mL(\theta), \theta' - \theta\rangle.
\end{align*}
We can use the proof of Lemma \ref{lem:deep_nc} to bound the other term as follows,
\begin{align*}
    &\left| \sum_{k=1}^{K}\sum_{j=1}^{d_u} \frac{\partial \mL}{\partial f_j(\theta[j]; \bar{z}_k)}\mathcal{R}(f_j, \theta[j], \theta'[j], \bar{z}_k)\right | \\
    &\le O(R^{1/3}H^{5/2}\sqrt{m\log m}) \sum_{k=1}^{K}\sum_{j=1}^{d_u} \left |\frac{\partial \mL}{\partial f_j(\theta[j]; \bar{z}_k)}\right | \|\theta'[j] - \theta[j]\|_F\\
    &\le  O(R^{4/3}H^{5/2}\sqrt{m\log m}) \sum_{k=1}^{K}\sum_{j=1}^{d_u} \left |\frac{\partial \mL}{\partial f_j(\theta[j]; \bar{z}_k)}\right |\\
    &\le O(R^{4/3}H^{5/2}KL_c'\sqrt{d_um\log m})
\end{align*}
We obtain that by Assumption \ref{assumption:control_stable}
\begin{align*}
        \mL(\bar{f}(\theta')) - \mL(\bar{f}(\theta)) \ge \langle \nabla_\theta \mL(\bar{f}(\theta)), \theta' - \theta\rangle - O(L_c R^{4/3} H^{9/2} K^6 W d_u \sqrt{d_x m} \log^{3/2} m),
\end{align*}
where $L_c' = O(K^5 L_c H^2 W \sqrt{d_x d_u} \log m)$ by Lemma \ref{lem:control_lip} and Corollary \ref{cor:lipschitz_control}.
\end{proof}


\end{document}